%% file: main.tex
\documentclass{article}

\usepackage[table]{xcolor} 

\usepackage{microtype}
\usepackage{graphicx}
\usepackage{subfigure}
\usepackage{booktabs} 
\usepackage{amsmath}
\usepackage{ifthen}
\usepackage{amsfonts}
\usepackage{lipsum}
\usepackage{makecell}
\usepackage[most]{tcolorbox}
\usepackage{dsfont}
\usepackage{bbold}
\usepackage[export]{adjustbox}
\usepackage{multirow}

\makeatletter
\DeclareFontEncoding{LS2}{}{\@noaccents}
\DeclareFontSubstitution{LS2}{stix}{m}{n}
\DeclareSymbolFont{arrows3stix}{LS2}{stixtt}{m}{n}
\DeclareMathSymbol{\bigwhitestar}{\mathord}{arrows3stix}{"A0}
\makeatother

\usepackage{pifont}
\newcommand{\cmark}{\ding{51}}%
\newcommand{\hypbox}[2]{%
\begin{tcolorbox}[colback=white!98!black,colframe=white!30!black,boxsep=1.1pt,top=6.75pt]%
\vspace{1.75pt}%
\textbf{#1}\\[-0.575em]
\noindent\makebox[\textwidth]{\rule{\textwidth}{0.4pt}}
\\[0.25em]
#2
\end{tcolorbox}
}

\usepackage{hyperref}

\usepackage[accepted]{icml2024}

\usepackage{amsmath}
\usepackage{amssymb}
\usepackage{mathtools}
\usepackage{amsthm}
\usepackage{xspace}
\usepackage{enumitem}
\usepackage[most]{tcolorbox}
\tcbuselibrary{skins}

\usepackage[capitalize,noabbrev]{cleveref}
\usepackage{macros}

\theoremstyle{plain}
\newtheorem{theorem}{Theorem}[section]
\newtheorem{proposition}[theorem]{Proposition}

\theoremstyle{definition}

\theoremstyle{remark}
\newtheorem{remark}[theorem]{Remark}

\usepackage[textsize=tiny]{todonotes}

\icmltitlerunning{The Platonic Representation Hypothesis}

\usepackage{inconsolata}

\begin{document}

\twocolumn[
\icmltitle{The Platonic Representation Hypothesis}



\icmlsetsymbol{equal}{*}

\begin{icmlauthorlist}
\icmlauthor{Minyoung Huh}{equal,mit}
\icmlauthor{Brian Cheung}{equal,mit}
\icmlauthor{Tongzhou Wang}{equal,mit}
\icmlauthor{Phillip Isola}{equal,mit}
\end{icmlauthorlist}

\icmlaffiliation{mit}{MIT}

\icmlcorrespondingauthor{Minyoung Huh}{minhuh@mit.edu}

\icmlkeywords{Machine Learning, Representation, Artificial Intelligence, Multimodality}

\vskip 0.3in
]



\printAffiliationsAndNotice{\icmlEqualContribution} 


\begin{abstract}
We argue that representations in AI models, particularly deep networks, are converging. First, we survey many examples of convergence in the literature: over time and across multiple domains, the ways by which different neural networks represent data are becoming more aligned. 
Next, we demonstrate convergence across data modalities: as vision models and language models get larger, they measure distance between datapoints in a more and more alike way. 
We hypothesize that this convergence is driving toward a shared statistical model of reality, akin to Plato's concept of an ideal reality. We term such a representation the \textit{platonic representation} and discuss several possible selective pressures toward it. Finally, we discuss the implications of these trends, their limitations, and counterexamples to our analysis.

\vspace*{0pt}
\begingroup%
\fontsize{8.5pt}{10pt}\selectfont%
\noindent\begin{tabular}{@{}lr@{}}
\textbf{\fontsize{9.5pt}{10pt}\selectfont Project Page:}\hspace{0.2em} & \hspace{-1.025em}\href{https://phillipi.github.io/prh/}{\texttt{phillipi.github.io/prh}}\\[1.1ex]
{\textbf{\fontsize{9.5pt}{10pt}\selectfont Code:}} & \hspace{-1.75em}%
{
\href{https://github.com/minyoungg/platonic-rep/}%
{\texttt{github.com/minyoungg/platonic-rep}}%
}
\end{tabular}%
\endgroup%
\end{abstract}



\section{Introduction}

AI systems are rapidly evolving into highly multifunctional entities. For example, whereas in the past we had special-purpose solutions for different language processing tasks (\eg, sentiment analysis, parsing, dialogue), modern large language models (LLMs) are competent at all these tasks using a single set of weights~\cite{srivastava2022beyond}. Unified systems are also being built across data modalities: instead of using a different architecture for processing images versus text, recent models, such as GPT4-V~\cite{achiam2023gpt}, Gemini~\cite{team2023gemini}, and LLaVA~\cite{liu2023llava}, handle both modalities with a combined architecture. 
More and more systems are built off of general-purpose pretrained backbones, sometimes called foundation models~\cite{bommasani2021opportunities}, that support a large range of tasks, including robotics~\cite{driess2023palm,brohan2023rt}, bioinformatics~\cite{ma2024segment}, and healthcare~\citep{steinberg2021language}.
In short, AI systems are becoming increasingly homogeneous in both their architectures and their capabilities.

\begin{figure}[t]
    \centering
    \hypbox{The Platonic Representation Hypothesis}{%
    Neural networks, trained with different objectives on different data and modalities, are converging to a shared statistical model of reality in their representation spaces.}
    \vspace{3pt}
    \includegraphics[width=0.85\linewidth]{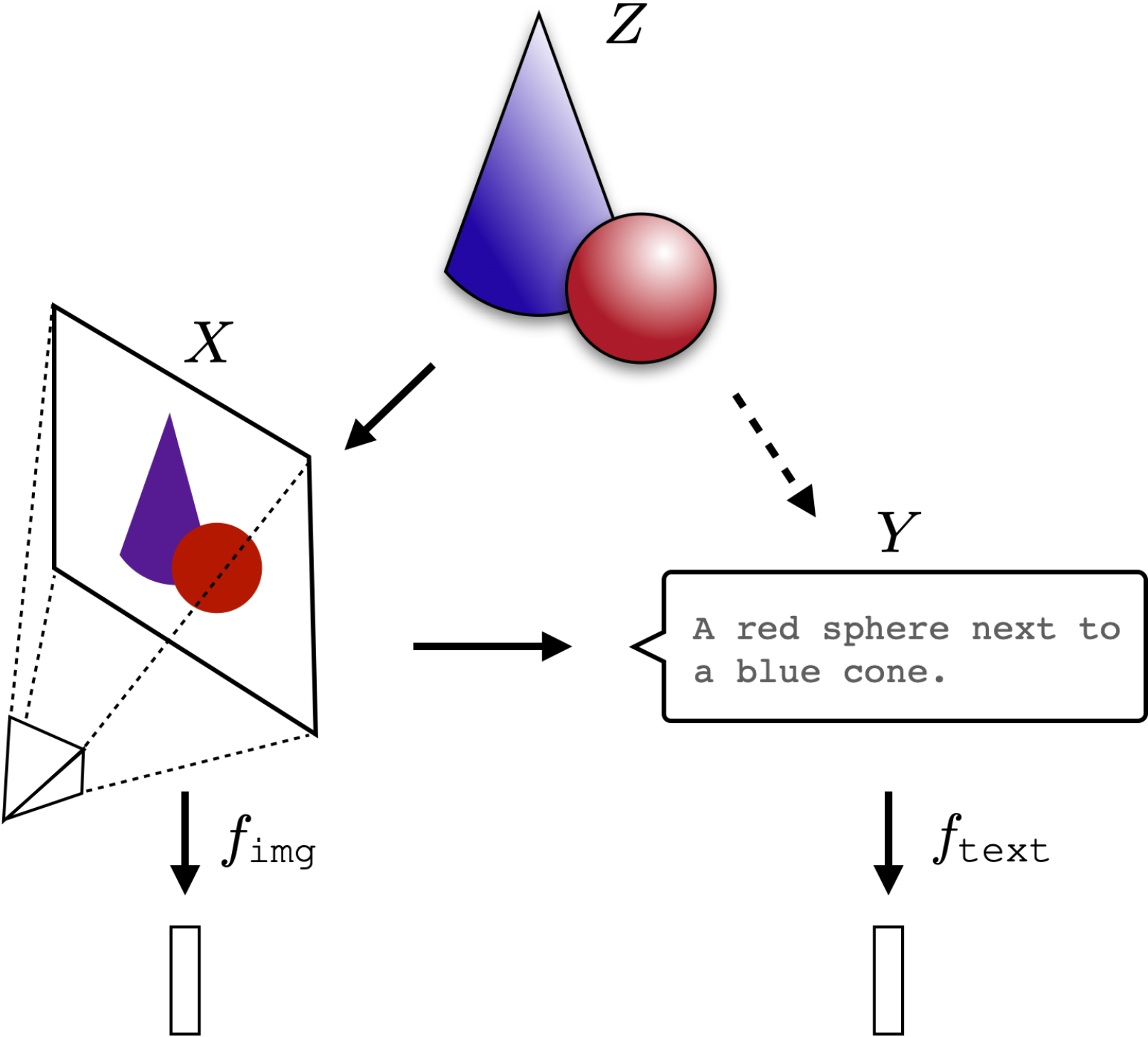}\vspace*{-5.5pt}
    \caption{\small \textbf{The Platonic Representation Hypothesis:} Images ($X$) and text ($Y$) are projections of a common underlying reality ($Z$). We conjecture that representation learning algorithms will converge on a shared representation of $Z$, and scaling model size, as well as data and task diversity, drives this convergence. 
    }\label{fig:platonic_rep}
\end{figure}

This paper explores one aspect of this trend: representational convergence. We argue that there is a growing similarity in how datapoints are represented in different neural network models. This similarity spans across different model architectures, training objectives, and even data modalities.

What has led to this convergence? Will it continue? And ultimately, where does it end?

Our central hypothesis, stated above in \Cref{fig:platonic_rep}, is that there is indeed an endpoint to this convergence and a principle that drives it: different models are all trying to arrive at a \textit{representation of reality}, meaning a representation of the joint distribution over events in the world that generate the data we observe. \Cref{fig:platonic_rep} conveys this hypothesis: there exists a real world (labeled $Z$), which we measure with various sensors, such as the camera shown to the left ($X$). Other \textit{projections} of these measurements, such as the textual description shown, can be produced from the first set of measurements or mediated by some other set of measurements, \eg, touch or other camera views (dotted arrow from $X$ to $Y$)\footnote{Touch could convey the shapes in this example but not the colors. This is an important limitation to our hypothesis that we discuss at several points in the paper: different sensors and views might capture different information, which may limit their potential to converge to identical representations.
}. 
Representation learning algorithms find vector embeddings that statistically model the various measurements and projections. The resulting vector embeddings are all derived from the underlying reality in $Z$ and thereby become aligned. As models are trained on more data and for more tasks, they require representations that capture more and more information about $Z$, and hence alignment toward $Z$ increases toward a convergent point as a function of scale.

We call this converged hypothetical representation the ``platonic representation'' in reference to Plato's Allegory of the Cave~\cite{plato_cave}, and his idea of an ideal reality that underlies our sensations. The training data for our algorithms are shadows on the cave wall, yet, we hypothesize, models are recovering ever better representations of the actual world outside the cave. This idea is not unique to Plato; our hypothesis is also related to the notion of ``convergent realism''~\cite{newton1981rationality,putnam1982three,doppelt2007reconstructing,hardin1982defense} in the philosophy of science (\ie, that science is converging on truth), and to many arguments that have been put forth in the representation learning literature (\eg, \citet{tian2020contrastive,zimmermann2021contrastive,richens2024robust,cao2021explanatory}).


Also closely related to our hypothesis is the ``Anna Karenina scenario'' described by \citet{bansal2021revisiting}, referring to the possibility that all well-performing neural nets represent the world in the same way. We discuss the evidence they give for this possibility in \Cref{sec:reps_are_converging}\footnote{Borrowed from \citet{tolstoy1877anna}, similar analogies have been made in other domains, such as the ``Anna Karenina principle'' popularized by \citet{diamond1998guns} to explain animal domestication.}. The platonic representation hypothesis refers to the situation where we are in an Anna Karenina scenario \textit{and} the ``happy representation'' that is converged upon is one that reflects a statistical model of the underlying reality. We discuss the potential nature of this statistical model in more detail in \Cref{sec:what_rep}.

\vspace{-3pt}
\section{Representations are converging}\label{sec:reps_are_converging}

\paragraph{Preliminaries}
We restrict our attention to representations that are \textit{vector embeddings}. We characterize such a representation by the similarity structure it induces, referred to as its kernel. Kernels are commonly used to assess representations~\cite{kornblith2019similarity, klabunde2023similarity}; this can be justified by the fact that they capture the relative structures among data samples, which are also the learning signal for many machine learning algorithms ~\cite{aronszajn1950theory,smola1998learning}. Following prior literature, we define \textit{representational alignment} as a measure of the similarity of the similarity structures induced by two representations, \ie, a similarity metric over kernels.
We give the mathematical definition of these concepts below:
\begin{itemize}[topsep=-1.5pt,itemsep=-1.5pt,leftmargin=10pt]
    \item A \textbf{representation} is a function $f\colon \mathcal{X} \rightarrow \mathbb{R}^n$ that assigns a feature vector to each input in some data domain $\mathcal{X}$. 
    \item A \textbf{kernel}, $K\colon \mathcal{X} \times \mathcal{X} \rightarrow \mathbb{R}$, characterizes how a representation measures distance/similarity between datapoints. $K(x_i,x_j) = \langle f(x_i), f(x_j) \rangle$, where $\langle {{}\cdot{}},{{}\cdot{}}\rangle$ denotes inner product, $x_i, x_j \in \mathcal{X}$ and $K \in \mathcal{K}$.
    \item A \textbf{kernel-alignment metric}, $m\colon \mathcal{K} \times \mathcal{K} \rightarrow \mathbb{R}$, measures the similarity between two kernels, \ie, how similar is the distance measure induced by one representation to the distance measure induced by another. Examples include Centered Kernel Distance (CKA)~\cite{kornblith2019similarity}, SVCCA~\cite{raghu2017svcca}, and nearest-neighbor metrics~\cite{klabunde2023similarity}.
\end{itemize}

In our experiments, we use a \emph{mutual nearest-neighbor metric} that measures the mean intersection of the $k$-nearest neighbor sets induced by two kernels, $K_1$ and $K_2$, normalized by $k$.
This metric is a variant of those proposed in~\citet{park2024quantifying}, \citet{klabunde2023similarity} and \citet{oron2017best}.
See~\app{sec:align-metric} for the exact definition and~\app{app:other-metrics} for comparisons with alternative alignment metrics.

Next, we explore several ways in which representations are converging. First, we argue that different neural networks are converging to aligned representations. Then, we show that this continues to hold across modalities, where image embeddings in vision models align with text embeddings in language models.

\begin{figure*}[ht!]
    \vspace{-2pt}
    \begin{minipage}[t]{0.65\textwidth}
        \hspace{-0.1in}
        \raisebox{-\height}{\includegraphics[width=0.47\linewidth, trim=0 0 0 0]{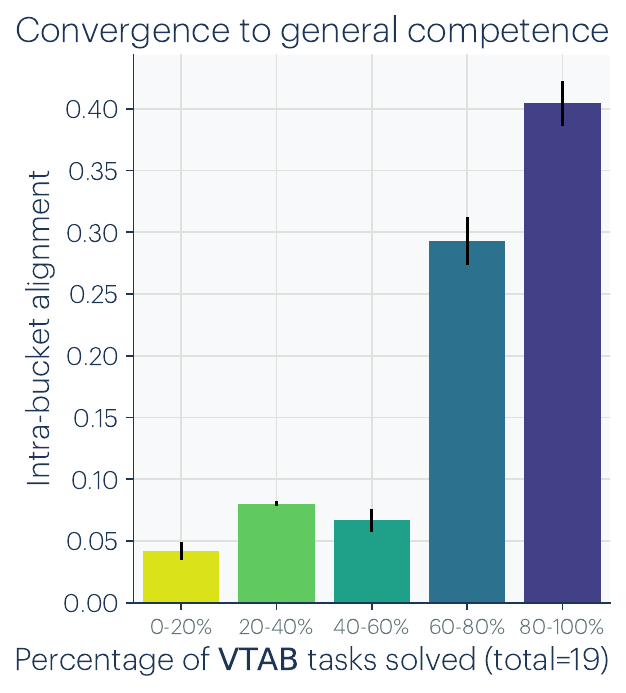}}
        \raisebox{-\height}{\includegraphics[width=0.504\linewidth, trim=5 0 11 0]{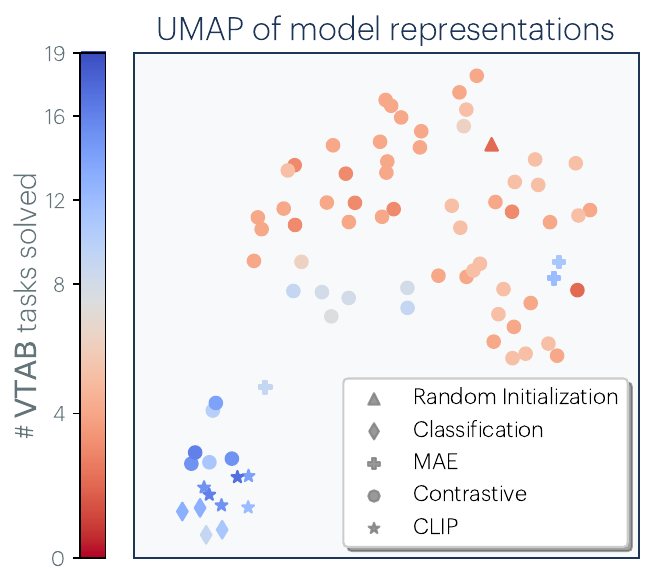}}
        \hfill
    \end{minipage}%
    \hfill
    \begin{minipage}[t]{0.35\textwidth}
        \vspace{-0.075in}
        \caption{%
            \small \textbf{VISION models converge as COMPETENCE increases:} We measure alignment among $78$ models using mutual nearest-neighbors on Places-365 \cite{zhou2017places}, and evaluate their performance on downstream tasks from the Visual Task Adaptation Benchmark (VTAB; \citet{zhai2019vtab}). \textbf{LEFT:} Models that solve more VTAB tasks tend to be more aligned with each other. Error bars show standard error. \textbf{RIGHT:} We use UMAP to embed \emph{models} into a 2D space, based on $\mathsf{distance} \triangleq -\log (\mathsf{alignment})$. More competent and general models (blue) have more similar representations.}
        \label{fig:vm_align}
    \end{minipage}
    \vspace{-16pt}
\end{figure*}

\subsection{Different models, with different architectures and objectives, can have aligned representations}


One indication of representational convergence is the rising number of systems built on top of pre-trained foundation models. These models are becoming standard backbones across a growing spectrum of tasks. Their versatility across numerous applications implies a level of universality in the way they represent data. 

While this trend implies convergence toward a relatively small set of foundation models, it does not imply that \textit{different} foundation models will arrive at the same representation. Yet that is what has been observed by several recent papers. 

\citet{lenc2015understanding} conducted one such study, in which they measured representational similarity through a technique called \textit{model stitching}. Given two models, $f$ and $g$, each composed of multiple layers ($ f = f_1 \circ \cdots \circ f_n $, $ g = g_1 \circ \cdots \circ g_m $), an intermediate representation from $f$ is integrated into $g$ via a learned affine stitching layer $ h $, resulting in a new stitched model $F = f_1 \circ \cdots \circ f_k \circ h \circ g_{k+1} \circ \cdots \circ g_m $. 
If $F$ has good performance, it indicates that $f$ and $g$ have compatible representations at layer $k$, up to the  transform $h$.


In their study,~\citet{lenc2015understanding} made two notable findings: (1) A vision model trained on ImageNet~\cite{russakovsky2015imagenet} can be aligned with a model trained on Places-365~\citep{zhou2017places} while maintaining good performance; (2) The early layers of these convolutional networks are more interchangeable than later layers. The first finding illustrates a level of data independence where distinct image datasets lead to similar representations. The second finding agrees with extensive research that oriented Gabor-like filters are common in both artificial and biological vision systems. This suggests a convergence to a similar initial layer of representation across various neural network architectures~\citep{olshausen1996emergence, krizhevsky2017imagenet}.
\citet{bansal2021revisiting} expanded on the idea of model stitching, showing that models trained using self-supervised objectives align closely with their supervised counterparts.



\citet{moschella2022relative} further demonstrated the feasibility of ``zero-shot'' model stitching without learning a stitching layer.
Despite the fact that different text models were trained on different modalities, they found that the models often
embed data in remarkably similar ways. In particular, they considered the kernel $K$ defined by learned representations and showed that $K$ serves as a bridge between models, allowing an encoder trained in one language, like English, to work effectively with a decoder in another, like French. 

\citet{dravid2023rosetta} extended this idea to individual neurons, and found ``Rosetta Neurons'' that are activated by the same pattern across a range of vision models. Such neurons form a common dictionary independently discovered by all models.

\begin{figure*}[t!]
    \centering
    
    \includegraphics[width=1.0\linewidth, trim=0 0 0 16]{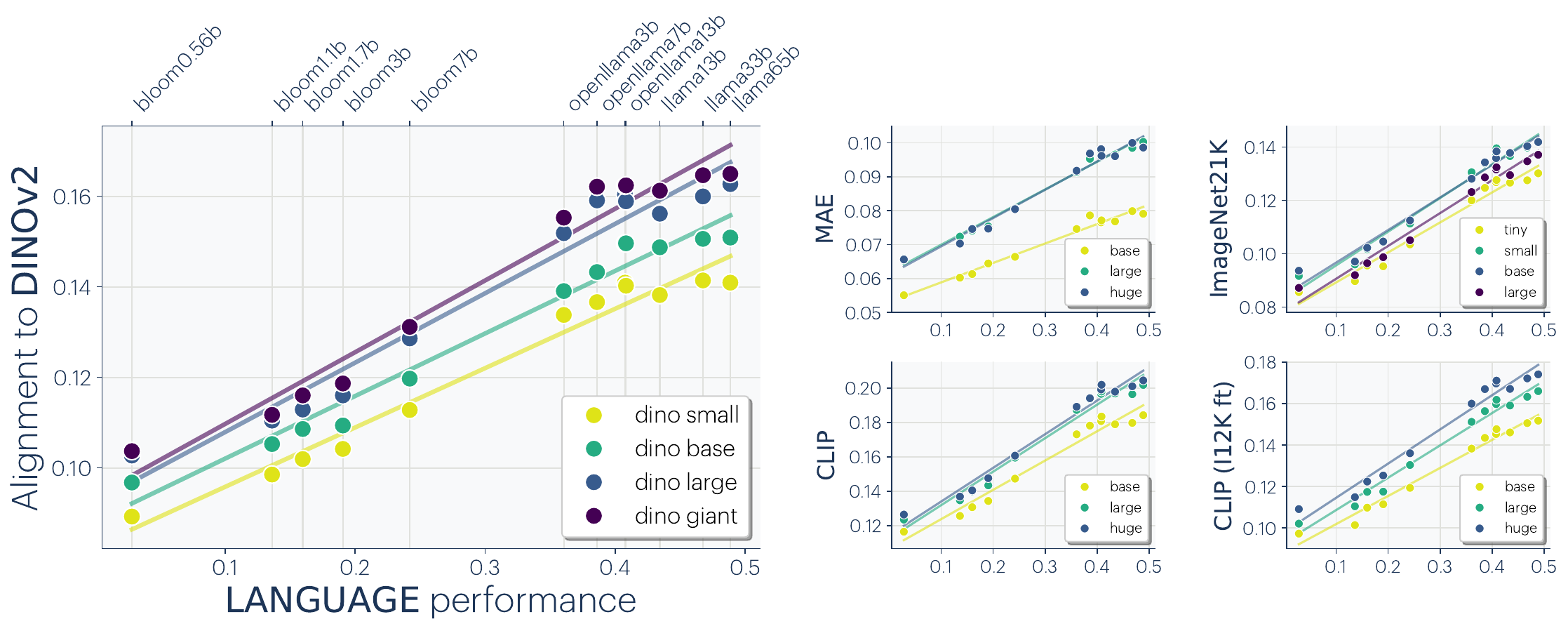}
    \vspace{-22pt}
    \caption{\small\textbf{LANGUAGE and VISION models align:} 
    We measure alignment using mutual nearest-neighbor on the Wikipedia caption dataset~(WIT)~\cite{srinivasan2021wit}. The x-axis is the language model performance measured over 4M tokens from the OpenWebText dataset~\cite{Gokaslan2019OpenWeb} (see \app{app:other-metrics} for plots with model names). 
    We measure performance using $1 - \texttt{bits-per-byte}$, where $\texttt{bits-per-byte}$ normalizes the cross-entropy by the total bytes in the input text string.
    The results show a linear relationship between language-vision alignment and language modeling score, where a general trend is that more capable language models align better with more capable vision models. 
    We find that CLIP models, which are trained with explicit language supervision, exhibit a higher level of alignment. However, this alignment decreases after being fine-tuned on ImageNet classification (labeled CLIP (I12K ft)).
    }%
    \label{fig:alignment_comparisons}
    \vspace{-4pt}
\end{figure*}

\subsection{Alignment increases with scale and performance}

\citet{kornblith2019similarity} and \citet{roeder2021linear} observed model alignment not only exists but also increases with model scale and dataset size. On CIFAR-10 classification, \citet{krizhevsky2009learning} found that larger models exhibit greater alignment with each other compared to smaller ones. 
Theoretically, \citet{balestriero2018spline} showed that models with similar outputs (\eg, as a result of having high performance) also have similar internal activations. 
With the continuing trend of models scaling up, this suggests model alignment will increase over time -- we might expect that the next generation of bigger, better models will be even more aligned with each other.
%

We expand upon this observation by evaluating the transfer performance of $78$ vision models. These models were trained with varying architectures, training objectives, and datasets~(detailed in~\Cref{sec:vision-vision-details}). In~\Cref{fig:vm_align} (left), we bin these models based on their average transfer performance on the VTAB dataset~\cite{zhai2019vtab}, and then measure the average kernel alignment of the models within each bin. 
The results indicate that models with high transfer performance form a tightly clustered set of representations, while models with weak performance have more variable representations. We further visualize this structure with UMAP~\citep{mcinnes2018umap} over models representation in~\Cref{fig:vm_align} (right). This suggests that models that are competent all represent data in a similar way. Echoing \citet{bansal2021revisiting} and \citet{tolstoy1877anna}, we might say: all strong models are alike, each weak model is weak in its own way.

The discussion so far indicates that various models are aligning toward a unified representation. But does the convergence extend to model weights? While models with different architectures might not have compatible weight spaces, there exists ample evidence that models with the same architecture will often converge to the same basin of weights~\cite{nagarajan2019uniform,garipov2018loss,lubana2023mechanistic}. This holds even for models with different initializations, up to permutations over weight space~\citep{ainsworth2022git}. Because of this, it is possible to merge separately trained models of the same architecture, and achieve some of the capabilities of all models in the mixture~\cite{stoica2023zipit,jordan2022repair,wortsman2022model}.

\begin{figure*}[ht]
    \centering
    \includegraphics[width=0.5\linewidth]{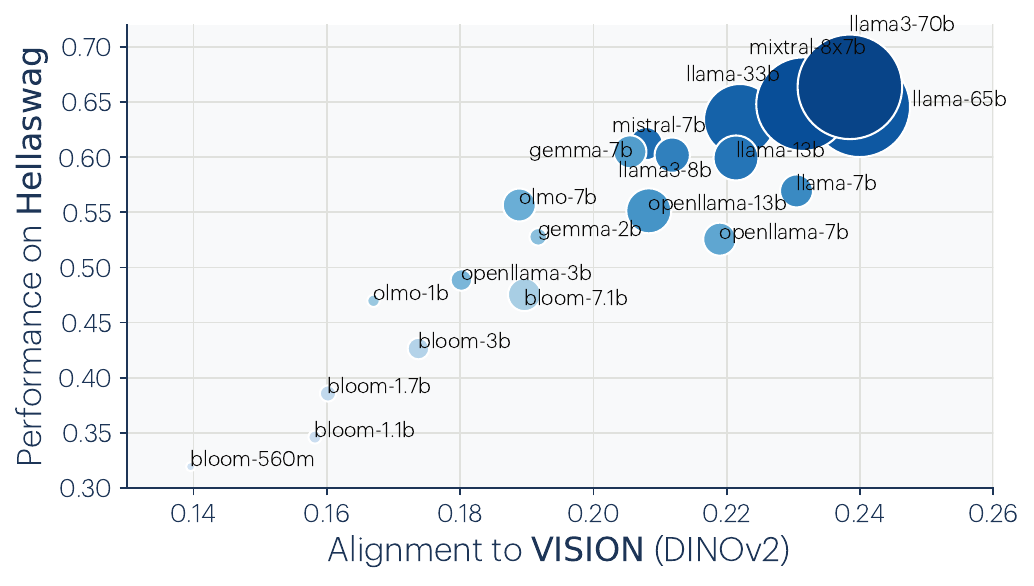}%
    \hfill%
    \includegraphics[width=0.5\linewidth]{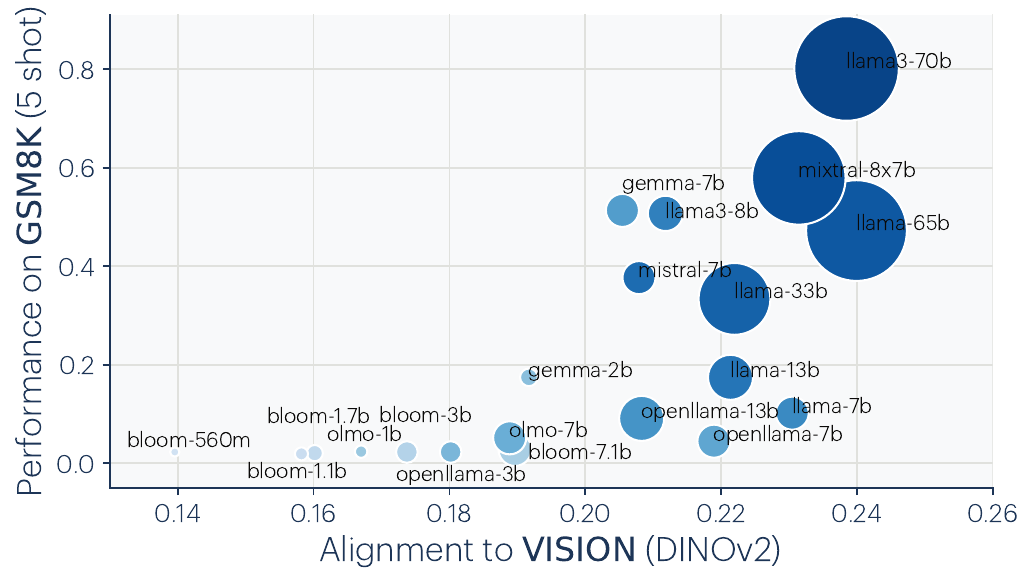}\\[-0.15in]
    \caption{\small\textbf{Alignment predicts downstream performance:} We visualize correlation between LLM alignment score to DINOv2~\cite{oquab2023dinov2} and downstream task performance on Hellaswag~(common-sense)~\cite{zellers2019hellaswag} and GSM8K~(math)~\cite{cobbe2021gsm8k}. LLMs are plotted with radii proportional to the size of the model, and color-coded by their rank order in language modeling scores ($1 - \texttt{bits-per-byte}$). We observe that models aligned more closely with vision also show better performance on downstream language tasks. For Hellaswag, there is a linear relationship with alignment score, while GSM8K exhibits an ``emergence''-esque trend. 
    } 
    \label{fig:downstream}
\end{figure*}

\subsection{Representations are converging across modalities}
Do models trained on different data modalities also converge?
Several works indicate that the answer is \emph{yes}. 

\citet{merullo2022linearly} extended model stitching to the cross-modal setting, finding that a single linear projection is sufficient to stitch a vision model to an LLM and achieve good performance on visual question answering and image captioning. \citet{koh2023grounding} showed that linear stitching can also work in the opposite direction, aligning text inputs to visual outputs. 
In fact, many recent language-vision models stitch pre-trained language and vision models together. For example, LLaVA~\cite{liu2023llava} demonstrated state-of-the-art results by projecting visual features into a language model with a 2-layer MLP.

Other works show further kinds of evidence of cross-modal synergy. \citet{achiam2023gpt} found that jointly training a language model with a vision model improves performance on language tasks, compared to training the language model on its own. \citet{sorscher2022neural} show a setting in which word embeddings of visual concept names can be isometrically mapped to image embeddings for those same concepts. In work concurrent to ours, \citet{maniparambil2024vision} show well-trained vision encoders on large datasets exhibit high semantic similarity with language encoders regardless of the training paradigm (supervised, self-supervised, or language-supervised). \citet{sharma2024vision} probed the visual knowledge of LLMs trained \textit{only} on language data, by converting images into code that an LLM can process. They found that LLMs have rich knowledge of visual structures, to the extent that decent visual representations can be trained on images generated solely by querying an LLM to produce code and rendering the response. In visual generation, LLMs show abilities to augment captions with visual structures (\eg, bounding boxes) and improve generation quality \citep{betker2023improving,lian2023llm,lian2023llmvideo,wu2023self}. Over other modalities, \citet{ngo2024language} showed auditory models are also roughly aligned with LLMs up to a linear transformation, and \citet{ng2023can} demonstrated the effectiveness of using pre-trained LLMs for facial motion prediction.

We set out to address these claims in a broader scope to determine whether models are indeed learning an increasingly modality-agnostic representation of the world. We sampled a variety of models trained either solely on vision or solely on language, and compared their representations as they became larger and more competent over many tasks.

In~\Cref{fig:alignment_comparisons}, we assess alignment between a suite of language models and vision models. So far we have only defined alignment for two kernels defined over the same input space. To measure cross-modal alignment, we use paired datasets to bridge the two modalities. For vision and text, we use the Wikipedia captions dataset $\{(x_i, y_i)\}_i$~\cite{srinivasan2021wit}, composed of images from Wikipedia ($x_i$) and their corresponding captions ($y_i$). We then measure alignment between a language model $f_{\texttt{text}}$ and a vision model $f_{\texttt{img}}$ as the alignment of the two following kernels:
\begin{align}
    K_\texttt{img}(i, j) = \langle f_{\texttt{img}}(x_i), f_{\texttt{img}}(x_j) \rangle\\
    K_\texttt{text}(i, j) = \langle f_{\texttt{text}}(y_i), f_{\texttt{text}}(y_j) \rangle.
\end{align}
Using this analysis, we find that the better an LLM is at language modeling, the more it tends to aligns with vision models, as shown in \Cref{fig:alignment_comparisons}. The converse effect also holds: the better a vision models is, the more it tends to align with LLMs. See \Cref{sec:vision-language-details} for more details.

\subsection{Models are increasingly aligning to brains}
\label{sec:models-and-minds}

Neural networks also show substantial alignment with biological representations in the brain~\cite{yamins2014performance}. This commonality may be due to similarities in the task and data constraints both systems are confronted with.
Even though the mediums may differ -- silicon transistors versus biological neurons -- the fundamental problem faced by brains and machines is the same: efficiently extracting and understanding the underlying structure in images, text, sounds, \etc~\cite{barlow1961possible, olshausen1997sparse}. \citet{sorscher2022neural} developed a theoretical framework for how the efficient extraction of novel concepts occurs for both the human visual system and deep networks. The tasks that the human visual system has been honed to perform through evolution -- like segmentation, detection, and whole-image classification -- are also the ones that we train our neural nets to perform. \citet{yamins2014performance} went as far as to title their work in the spirit that performance over such tasks implies brain alignment. \citet{antonello2024predictive} posited that it is less the particular task and more the generality of the representations that explain their alignment with biological representations. Further, \citet{conwell2022can} showed that training data plays a large role in alignment. Psychophysical studies have also shown agreement between how humans perceive visual similarity and how models do, even when the models are trained on tasks, such as self-supervised prediction, that are seemingly unrelated to mimicking human perception~\cite{zhang2018unreasonable}.


\subsection{Does alignment predict downstream performance?}
If models are converging towards a more accurate representation of reality, we expect that alignment should correspond to improved performance on downstream tasks. \Cref{fig:downstream} supports this hypothesis by demonstrating improved performance on commonsense reasoning (Hellaswag; \citet{zellers2019hellaswag}) and mathematical problem solving (GSM8K; \citet{cobbe2021gsm8k}) as alignment improves. 

\section{Why are representations converging?}\label{sec:what-why-converge}


Modern machine learning models are generally trained to minimize the empirical risk with possible implicit and/or explicit regularization: \begin{equation*}
    {
    \color{gray}
    \overbracket[1pt]{\color{black}f^*}^{\mathclap{\textsf{trained model}}}
    }{}= \mathbox[model]{\argmin}_{f \in {
    \color{gray}
    \underbracket[1pt]{\scriptsize\mathbox[model]{\color{black}\mathcal{F}}}_{\mathclap{\scriptstyle\textsf{function class}}}}
    }\mathbb{E}_{x \sim {\scriptsize\mathbox[task]{\mathsf{dataset}}}}[ {
    \color{gray}
    \overbracket[1pt]{\mathbox[task]{\color{black}\mathcal{L}}}^{\mathclap{\textsf{training objective}}}}(f, x)] + {
    \color{gray}
    \underbracket[1pt]{\mathbox[bias]{\color{black}\mathcal{R}}}_{\mathclap{\textsf{regularization}}}}(f)
\end{equation*}
In the following sections, we lay out how each colored component in this optimization process potentially plays a role in facilitating representational convergence. 

\begin{figure*}[t]
    \centering
    \vspace{3pt}
    \includegraphics[width=0.945\linewidth, trim=0 0 12 0]{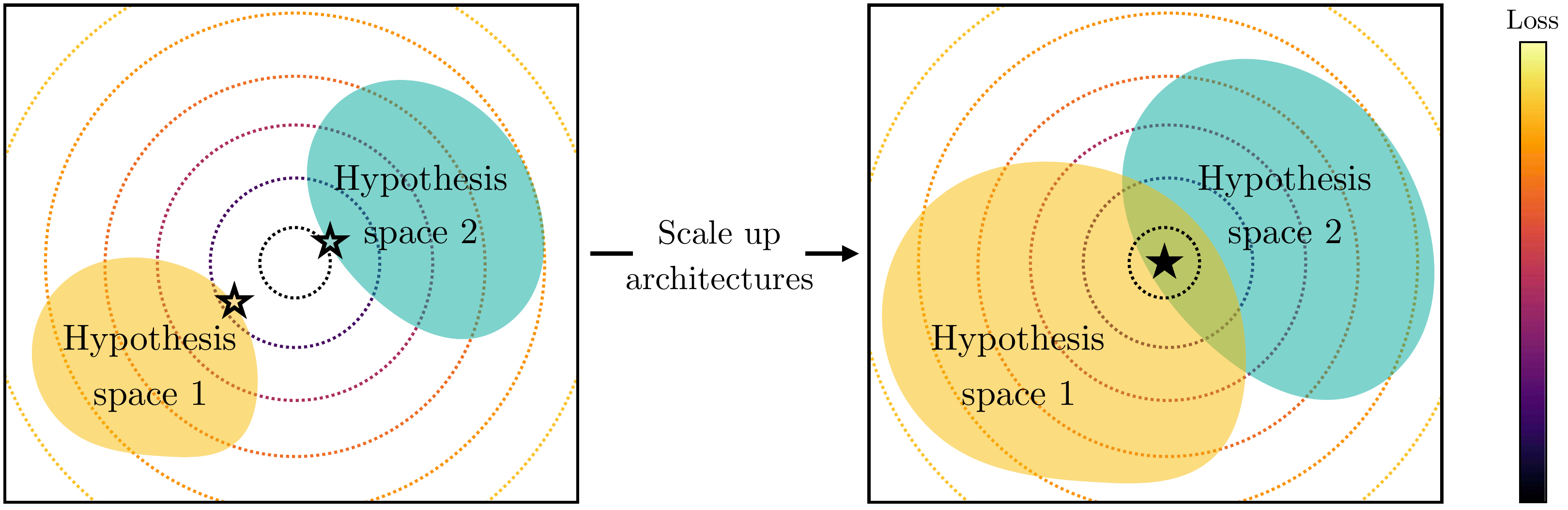}\\[-8pt]
    \caption{\small \textbf{The Capacity Hypothesis:} If an optimal representation exists in function space, larger hypothesis spaces are more likely to cover it. \textbf{LEFT:} Two small models might not cover the optimum and thus find \textit{different} solutions (marked by outlined \scalebox{1.25}{$\bigwhitestar$}). \textbf{RIGHT:} As the models become larger, they cover the optimum and converge to the same solution (marked by filled \scalebox{1.05}{$\bigstar$}).}
    \label{fig:hypothesis_space_overlap}%
    \vspace{-5pt}
\end{figure*}

\begin{figure}[t]
    \centering
    \vspace{3pt}
    \includegraphics[width=0.825\linewidth]{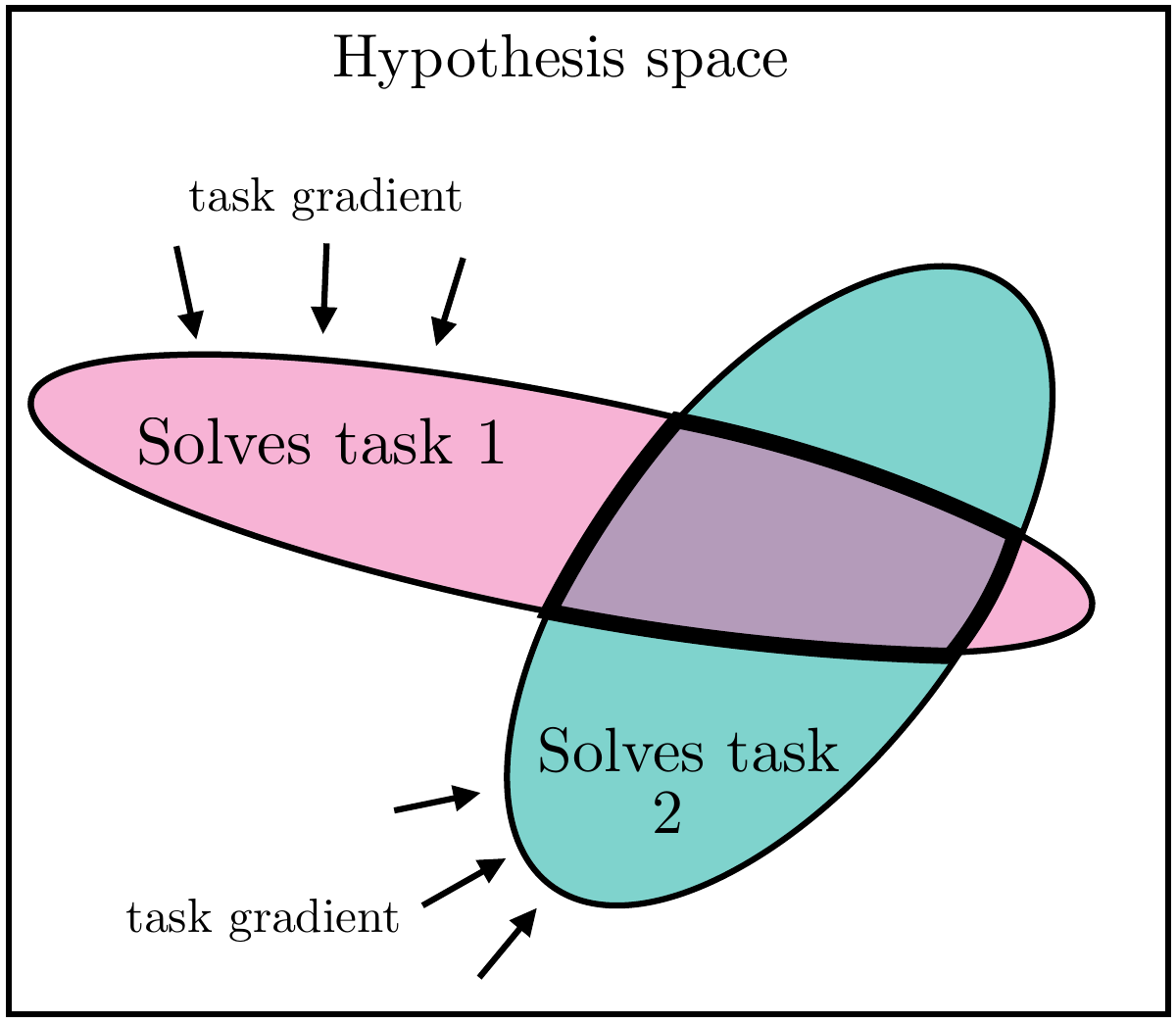}\\[-7pt]
    \caption{\small \textbf{The Multitask Scaling Hypothesis:} Models trained with an increasing number of tasks are subjected to pressure to learn a representation that can solve all the tasks.} \label{fig:multitask_hypothesis}
\end{figure}

\subsection{Convergence via \textbox[task]{Task Generality}}
\label{sec:multitask_scaling_hypothesis}


Each training datapoint and objective (task) places an additional constraint on the model. As data and tasks scale,
the volume of representations that satisfy these constraints must proportionately grow smaller, as visualized in Figure \ref{fig:multitask_hypothesis} and stated below: 
\hypbox{The Multitask Scaling Hypothesis}{%
There are fewer representations that are competent for $N$ tasks than there are for $M<N$ tasks. As we train more general models that solve more tasks at once, we should expect fewer possible solutions.
}

This has been previously termed as the Contravariance principle by~\citet{cao2021explanatory}, which states that the set of solutions to an easy goal is large, while the set of solutions to a challenging goal is comparatively smaller. Moreover, we argue that this narrower solution set also generalizes better. As data scales, models that optimize the empirical risk $\mathbb{E}_{x \sim {\scriptsize \mathbox[task]{\mathsf{dataset}}}}[ {{{\mathcal{L}}}}(f, x)]$ also improve on the population risk $\mathbb{E}_{x \sim {\scriptsize \mathbox[task]{\mathsf{reality}}}}[ {{{\mathcal{L}}}}(f, x)]$, and become better at capturing statistical structures of the true data generating process ($\mathsf{reality}$).

Recent work has demonstrated a power law relationship between data scale and model performance~\cite{hestness2017deep}. This implies that with enough data (\eg, consisting of the entire internet and all offline scientific measurements) one ought to converge to a very small solution set with irreducible error -- the inherent epistemic uncertainty of the world. 
As more models are trained on internet-scale data, the set of solutions that satisfies all data constraints must become relatively small. 

In addition to data-scaling, many modern representation learning objectives $\mathbox[task]{\mathcal{L}}(f, x)$ directly optimize for multi-task solving. Contrastive learning finds a distance structure over data samples that optimizes many classification tasks \citep{arora2019theoretical,tongzhouw2020hypersphere,tian2020rethinking}. Masked Autoencoders \citep{he2021masked} optimize randomly sampled reconstruction tasks. In fact, autoregressive language modeling can also be seen as optimizing a diverse set of tasks \citep{radford2019language}. Such multi-task objectives may be more effective than single-task ones (\eg, ImageNet classification) due to the fact that they impose more task constraints on the representation, leading to a smaller and higher-quality solution space \citep{chen2020simple,he2020momentum,radford2017learning,radford2019language}.




\subsection{Convergence via \textbox[model]{Model Capacity}}\label{sec:capacity_hypothesis}


Suppose there is a globally optimal representation for standard learning objectives. Then, under sufficient data, \textit{scaling} a model (\ie, using larger function classes 
\mathbox[model]{\mathcal{F}}), as well as \textbox[model]{improved optimization}, should be more effective at finding better approximations to this optimum, as illustrated in \Cref{fig:hypothesis_space_overlap}. 
With the same training objective, larger models, even of different architectures, will thus tend to converge toward this optimum. When different training objectives share similar minimizers, larger models are better at finding these minimizers, 
and will train to similar solutions over the training tasks. We summarize this hypothesis as follows:
\hypbox{The Capacity Hypothesis}{%
Bigger models are more likely to converge to a shared representation than smaller models. 
}



\subsection{Convergence via \textbox[bias]{Simplicity Bias}}\label{sec:simplicity_bias_hypothesis}

Arriving at the same mapping on the \textit{training data} does not prohibit the models from developing distinct internal representations. It is not unreasonable to posit that the representations used to detect a dog in a 1M parameter model could be quite different than that used by a 1B parameter model. What would stop a billion-parameter (and counting) model from learning an overly complicated and distinct representation? One key factor might be simplicity bias:

\hypbox{The Simplicity Bias Hypothesis}{
Deep networks are biased toward finding simple fits to the data, and the bigger the model, the stronger the bias. Therefore, as models get bigger, we should expect convergence to a smaller solution space.}




\vspace{2pt}
Such simplicity bias could be coming from explicit regularization $\mathbox[bias]{\mathcal{R}(f)}$ commonly used in deep learning (\eg, weight decay and dropout). However, even in the absence of external influences, deep networks naturally adhere to Occam's razor, \textbox[bias]{implicitly favoring simple solutions}that fit the data~\cite{solomonoff1964formal,gunasekar2018implicit,arora2019implicit,valle2018deep, huh2023simplicitybias,dingle2018input, goldblum2023no}. 
Figure \ref{fig:simplicity_hypothesis} visualizes how simplicity bias can drive convergence.

\begin{figure}[t]
    \centering
    \vspace{3pt}
    \includegraphics[width=0.825\linewidth]{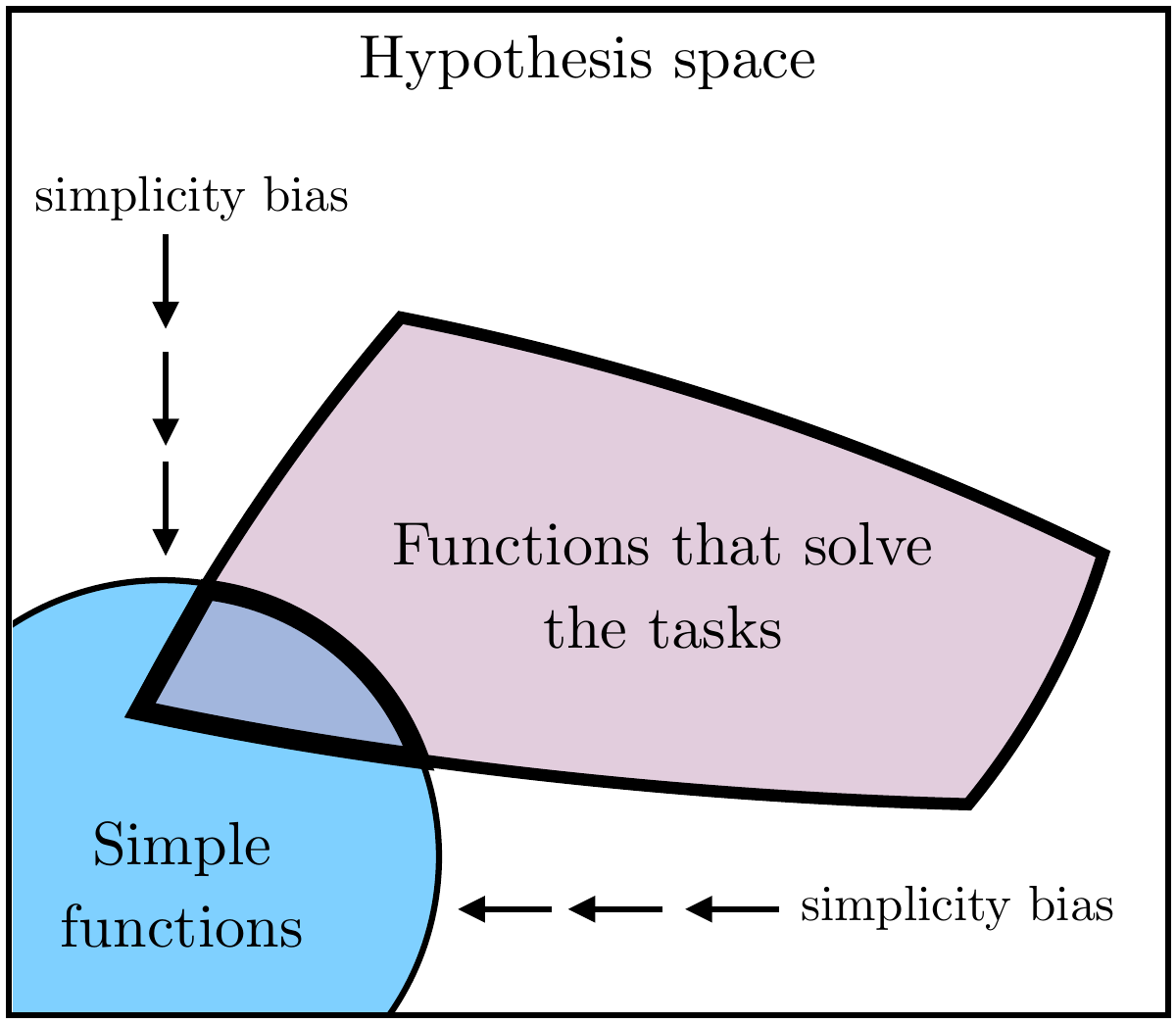}\\[-7pt]
    \caption{\small \textbf{The Simplicity Bias Hypothesis:} Larger models have larger coverage of all possible ways to fit the same data. However, the implicit simplicity biases of deep networks encourage larger models to find the simplest of these solutions.}
    \label{fig:simplicity_hypothesis}
    \vspace{2pt}
\end{figure}
\section{What representation are we converging to?}\label{sec:what_rep}


By now, we hope to have convinced the reader that task and data pressures, combined with increasing model capacity, can lead to convergence. We next turn our attention to \textit{what} exactly is the endpoint of all this convergence. 

Our central hypothesis, stated in~\Cref{fig:platonic_rep}, is that the representation we are converging toward is a statistical model of the underlying reality that generates our observations. Consistent with the multitask scaling hypothesis, such a representation would naturally be useful toward many tasks (or at least toward any task grounded in reality). Additionally, this representation might be relatively simple, assuming that scientists
are correct in suggesting that the fundamental laws of nature are indeed simple functions \citep{gell1995quark}, in line with the simplicity bias hypothesis.

But what exactly do we mean by ``a statistical model of the underlying reality.'' In this section, we formalize one definition with concrete mathematical statements. \emph{Importantly}, this section should be read as just one concrete candidate for the form of the platonic representation; other candidates could be arrived at from other modeling assumptions. 



\subsection{An idealized world}
We consider a world that works as follows, consistent with the cartoon in \Cref{fig:platonic_rep}. The world consists of a sequence of $T$ discrete events, denoted as $\mathbf{Z} \triangleq [z_1, \ldots, z_T]$, sampled from some unknown distribution $\mathbb{P}(\mathbf{Z})$. Each event can be observed in various ways. An observation is a bijective, deterministic function $\texttt{obs}: \mathcal{Z} \rightarrow \cdot{}\,$ that maps events to an arbitrary measurement space, such as pixels, sounds, mass, force, torque, words, etc. 
Later, in \Cref{sec:limitations}, we discuss limitations and potential extensions to continuous and unbounded worlds, and stochastic observations, that could yield a model that better reflects real learning scenarios.

One can think of an event as corresponding to the state of the world at some point in time\footnote{Here we only analyze temporal sequences, but note that the same could be done with respect to events laid out in space instead.}, but it is also fine to simply consider an event as any variable that indexes observations, with no further physical meaning\footnote{This latter interpretation may be more consistent with Plato's intent. Scholars have argued that his allegory of the cave rejects any notion of a true world state~\cite{nettleship1897lecturesplato}. Instead, we could say that the joint distribution of observation indices is \textit{itself} the platonic reality.}. 

In this idealized world, knowing $\mathbb{P}(\mathbf{Z})$ would be useful for many kinds of predictions; this would constitute a world model over the events that cause our observations~\citep{werbos1987learning,ha2018world,richens2024robust}. We will next show that a particular representation of $\mathbb{P}(\mathbf{Z})$ is recovered by certain contrastive learners.

\begin{figure*}[ht]
    \centering
    \includegraphics[width=0.99\linewidth,trim=174 352 188 295,clip]{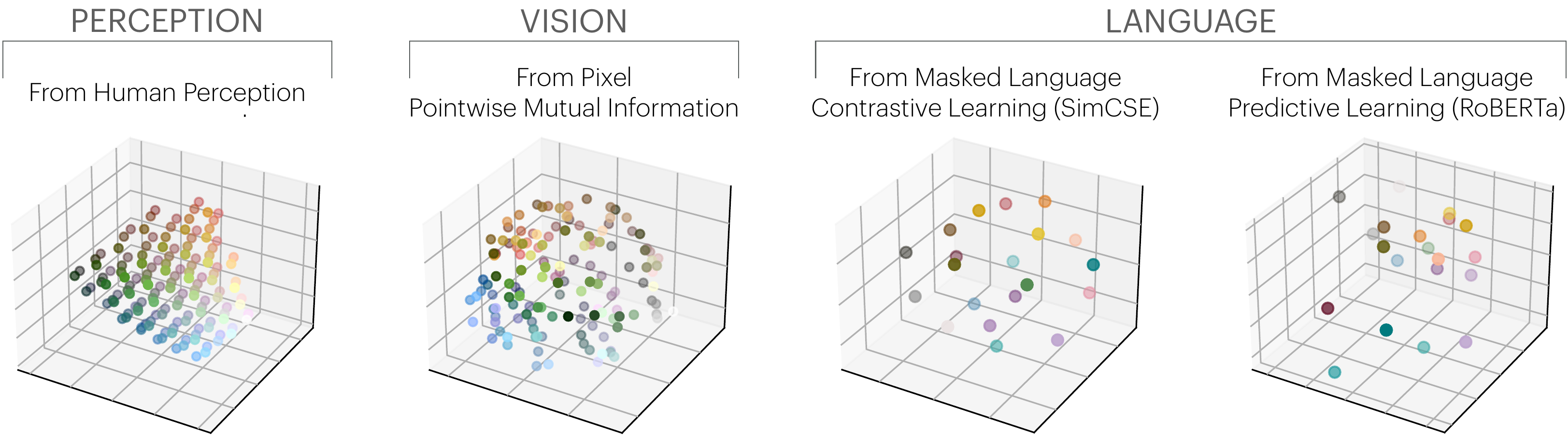}\\[-5pt]
    \caption{\small\textbf{Color cooccurrence in VISION and LANGUAGE yields perceptual organization:} Similar representations of color are obtained via, \textbf{from LEFT to RIGHT}, the perceptual layout from CIELAB color space,  cooccurrence in CIFAR-10 images, and language cooccurrence modeling (\citet{gao2021simcse,liu2019roberta}; computed roughly following \citet{abdou2021can}). Details in \Cref{sec:color_cooccurrences}.
    }
    \label{fig:color_pAB}
\end{figure*}

\subsection{A family of contrastive learners converge to a representation of $\mathbb{P}(\mathbf{Z})$}
\label{sec:simple-contra-kpmi}

Consider a contrastive learner that models observations that \textit{cooccur} together. For simplicity, we ground our discussion with the following definition of the \textit{cooccurrence probability}, $\Pco$, of two observations $x_a$ and $x_b$ both occurring within some  window $T_\mathsf{window}$: 
\begin{align}
    \Pco(x_a, x_b) \hspace{0.1in} \propto \hspace{-0in} \sum_{(t, t') \colon \abs{t-t'} \leq T_\mathsf{window}} \hspace{-0.2in} \mathbb{P}(X_t = x_a, X_{t'} = x_b).\nonumber
\end{align}
Analogously, we can define $\Pco$ for $\mathbf{Z}$ and other observation modalities. Note that $\Pco$ is symmetric.

Consider \emph{positive pairs} as two observations nearby in time (sampled from $\Pco$) and \emph{negative pairs} as observations drawn from any point in time (sampled independently from the marginal). Our contrastive learner tries to classify if a pair is positive or negative by learning a representation $f_X \colon X \rightarrow \mathbb{R}^d$ such that the dot-product kernel approximates the log odds ratio up to some offset:
\begin{align}
    \langle f_X(x_a), f_X(x_b) \rangle 
    & \approx \log \frac{\mathbb{P}(\texttt{pos} \given x_a, x_b)}{\mathbb{P}(\texttt{neg} \given x_a, x_b)} + \tilde{c}_X(x_a) \\
    & = \log \frac{\Pco(x_a \given x_b)}{\Pco(x_a)} + c_X(x_a) \\
    & =  \Kpmi(x_a, x_b) + c_X(x_a), \label{eqn:contr-pmi}
\end{align}
where $\Kpmi$ is the pointwise mutual information (PMI) kernel, and $c_X(x_a)$ is constant in $x_b$. We note that this is a common setting for self-supervised contrastive learners with NCE objectives~\cite{gutmann2010noise,oord2018representation}, including SimCLR~\cite{chen2020simple} and SimCSE~\cite{gao2021simcse}.  (See \citet{oord2018representation} and \Cref{sec:analysis_contrastive-pmi} for detailed derivations.)

Under mild conditions that the world is smooth enough (see \Cref{sec:analysis_contrastive-exact-repr}), a choice of $f_X$ can exactly represent $\Kpmi$:
\begin{align}
    \langle f_X(x_a), f_X(x_b) \rangle &= \Kpmi(x_a,x_b) + c_X,
\end{align}
where we observed that $c_X(x_a)$ from \Cref{eqn:contr-pmi} must be a constant since both sides are symmetric.
%

Therefore, the contrastive learners we consider are minimized by a representation $f_X$ whose kernel is $\Kpmi$ (up to a constant offset). With sufficient data and optimization, we will observe convergence to this point.

Thus we have convergence to a representation of the statistics of $X$, but what about $Z$? Recall that our idealized world consists of \textit{bijective} observation functions, which, over discrete random variables, preserve probabilities. So we have:
\begin{align*}
    \Pco(x_a, x_b) &= \Pco(z_a, z_b)\\
    \Kpmi(x_a, x_b) &= \Kpmi(z_a, z_b),
\end{align*}
where we use $\Pco$ and $\Kpmi$ in a modality-agnostic way to emphasize that different modalities share the same these quantities.



All these arguments hold not just for $X$ but also for $Y$ (or any other bijective, discrete modality), implying:
\begin{align}
    \Kpmi(z_a, z_b) 
    & = \langle f_X(x_a), f_X(x_b) \rangle - c_X \\
    & =     \langle f_Y(y_a), f_Y(y_b) \rangle  - c_Y.
\end{align}
Therefore, for any modality in our idealized world, we observe representational convergence to the same kernel, which represents certain pairwise statistics of $\mathbb{P}(\mathbf{Z})$.

This analysis suggests that certain representation learning algorithms may boil down to a simple rule: \textit{find an embedding in which similarity equals PMI}. We note that this idea is consistent with prior works that have used PMI as a similarity measure for clustering in vision and language (\eg, ~\citet{crisp_boundaries, isola_thesis, isola_cooc, chambers2008unsupervised}).

\paragraph{A study in color}
We conduct a case study to verify that convergence does happen on real data.
\citet{abdou2021can} discovered that color distances in learned language representations, when trained to predict cooccurrences in \emph{text} \citep{devlin2018bert}, closely mirror human perception of these distances, which we reproduce in \Cref{fig:color_pAB} with both contrastive and predictive models. Interestingly, they noted an increasing similarity as models scale larger and become better at modeling \emph{text} cooccurrences. In \Cref{fig:color_pAB}, we also learn representations of color based on $\Kpmi$ from cooccurrences in \emph{images}. 
Indeed, learning cooccurrence statistics in either domain recovers roughly the \emph{same} perceptual representation. Details of this experiment are described in \Cref{sec:color_cooccurrences}. 

We believe that our simple model encapsulates essential aspects of complex real-world systems, and offers a path toward understanding the representation that models are converging to---a unified model that is proficient across various domains and modalities, grounded in the statistical properties of the underlying world. \Cref{sec:limitations} further elaborates some limitations.



\section{What are the implications of convergence?}\label{sec:implications}

\paragraph{Scaling is sufficient, but not necessarily efficient} Our arguments are roughly in line with the claim that ``scale is all you need'' to reach high levels of intelligence. We have argued that as resources are scaled (\# parameters, \# datapoints, \# flops), representations are converging, regardless of other modeling choices and even data modality. 
Does this mean that scale is all that matters? Not quite: different methods can scale with different levels of \textit{efficiency} \citep{hestness2017deep,kaplan2020scaling}, and successful methods must still satisfy some general requirements (\eg, be a consistent estimator, model pairwise statistics of $\mathbb{P}(\mathbf{Z})$).

\paragraph{Training data can be shared across modalities} Suppose you have access to $N$ images and $M$ sentences, and want to learn the best representation. If there is indeed a modality-agnostic platonic representation, then \emph{both} image and language data should help find it.
The implication is that if you want to train the best vision model, you should train not just on  $N$ images but also on $M$ sentences. This is already becoming common practice~\cite{achiam2023gpt, radford2021learning}. Many vision models are finetuned from pre-trained LLMs. The other direction is less common, but also is implied by our hypothesis: if you want to build the best LLM, \textit{you should also train on image data}. Indeed, \citet{achiam2023gpt} showed that training on images improved performance on text. 
In theory, there should be some conversion ratio: a pixel is worth $a$ words for training LLMs, and a word is worth $b$ pixels for training vision models.

\paragraph{Ease of translation and adaptation across modalities}
When two representations are aligned, transitioning from one to the other should be a simple function that's easily obtained. Our hypothesis could explain the phenomenon that conditional generation is easier than unconditional~\citep{mirza2014conditional,liu2020selfconditioned,sauer2022styleganxl}, as the data we condition on may have the same platonic structure as the data we are generating. In line with this, recent work has found that representation-conditioning is even easier \citep{RCG2023}. Similarly, representational convergence could act as a bridge that lets us find mappings between domains even without paired data; this may underlie the success of unpaired translation in vision \citep{CycleGAN2017,shi2024diffusion,xie2022unsupervised} and language \citep{tranfeature2017tran,lample-etal-2018-phrase}. 
We emphasize that this doesn't mean that models trained on a single modality (\eg, language) can immediately process raw data from another (\eg, vision). What makes them adaptable to the new modalities is that they share a common modality-agnostic representation, and can readily process \emph{representations} of new modalities. 
Furthermore, this implies that language models would achieve some notion of grounding in the visual domain 
even in the absence of cross-modal data\footnote{
In 1688, William Molyneux asked if a person born blind, upon gaining sight, could distinguish shapes by vision alone~\citep{locke_molyneaux}. Our arguments suggest they could not do so immediately, but after some visual experience, they could easily map shapes to their prior touch-based representations. Empirical data supports this, showing that congenitally blind children given sight can quickly learn these abilities~\citep{held2011newly}. 
}.
%
The primary advantage of cross-modal data could then simply be sample efficiency. 

\paragraph{Scaling may reduce hallucination and bias} 
A prominent shortcoming of current LLMs is their propensity to hallucinate, or output false statements. If models are indeed converging toward an accurate model of reality, and scale powers this convergence, then we may expect hallucinations to decrease with scale. Of course, our hypothesis is conditioned on the training data for future models constituting a sufficiently lossless and diverse set of measurements. This may not come to pass, but it is an implication of our hypothesis worth pointing out. A similar argument can be made about certain kinds of bias. It has been shown that large models can exacerbate existing biases present in their training data~\citep{hall2022systematic}. Our hypothesis implies that, while this may be true, we should expect \textit{larger} models to amplify bias \textit{less}. This does not mean bias will be removed, rather that the model's biases will more accurately reflect the data's biases, rather than exacerbating them.

\section{Counterexamples and limitations}\label{sec:limitations}


\paragraph{Different modalities may contain different information}

One immediate objection to our hypothesis is: what about the information that is unique to a given modality? Can language really describe the ineffable experience of watching a total solar eclipse? Or, how could an image convey the a concept like ``I believe in the freedom of speech,'' which is easy to write in English? Two different models cannot converge to the same representation if they have access to fundamentally different information. 

More precisely, our mathematical argument in \Cref{sec:what_rep} only strictly holds for bijective projections of $\mathbf{Z}$, so that the information in all the projections is equivalent to the information in the underlying world. This will not hold true for either lossy or stochastic observation functions. Nonetheless, similar arguments have been made theoretically and empirically that cooccurrence relations are learned by practical contrastive \citep{tongzhouw2020hypersphere,zimmermann2021contrastive} and predictive learners \citep{papyan2020prevalence,roeder2021linear}.  
\citet{lu2021pretrained} and \citet{mirchandani2023large} also showed that models trained to autoregressively generate text also capture statistical relations in many other modalities, including symbolic reasoning, 
vision, protein folding, and robotics.

\begin{figure}[t!]
    \centering
    \includegraphics[width=\linewidth]{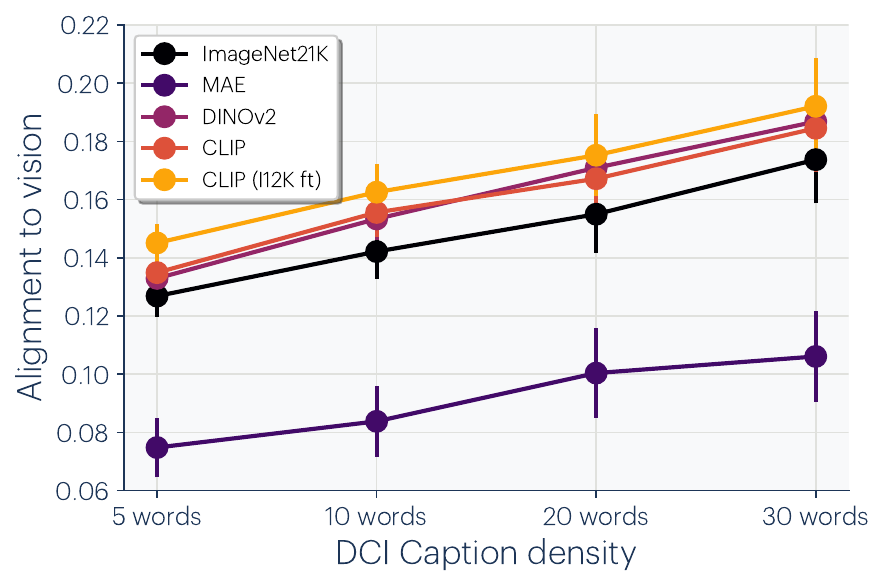}%
    \vspace{-11pt}%
    \caption{\small \textbf{Increasing caption density improves alignment:} We vary caption length using the Densely-Captioned-Images (DCI) dataset~\cite{urbanek2023picture}. Starting from a dense caption, we used LLaMA3-8B-Instruct~\cite{meta2024llama3} to summarize and generate coarse-grained captions. We compute the average alignment score across all vision and language models with standard deviation measured over the language models we evaluated. With denser captions, the mapping may become more bijective, leading to improved language-vision alignment scores.}
    \label{fig:caption_density}
\end{figure}


A more nuanced version of our hypothesis will need to be developed to handle the case of non-bijective observations and abstract concepts. A starting point could be: different models will converge to the same representation \textit{when the input signals are sufficiently high information and the models are sufficiently high capacity}; when they are not, the lower-information representation will only align with the higher-information one up to a level capped by the mutual information between the input signals and by the capacity of each model. This cap might or might not be practically important. Popular representations like CLIP are explicitly optimized to only capture the shared information between vision and language, yet are highly successful on many pure vision tasks. 
We perform a preliminary test of the effect of information level in \Cref{fig:caption_density} (detailed in \Cref{sec:caption_density}), and find that the more descriptive (higher information) a caption is, the better its LLM representation aligns with the visual representation of the corresponding image.

\paragraph{Not all representations are presently converging}

Our argument has mainly focused on two modalities: vision and language. While we do expect other modalities will follow similar trends, we have yet to see the same level of convergence across all domains. For example, in robotics there is not yet a standardized approach to representing world states in the same way as there is for representing images and text. One limitation lies in the hardware used in robotics, which is often expensive and slow. This creates a bottleneck in the quantity and diversity of training data.

\paragraph{Sociological bias in producing AI models}
Researcher bias and collective preferences within the AI community have shaped the trajectory of model development. 
There is often an explicit or implicit goal of designing AI systems that mimic human reasoning and performance, and this could lead to convergence toward human-like representations even if other kinds of intelligence are in fact possible. 
Additionally, the ``hardware lottery''~\cite{hooker2021hardware} suggests that the success of AI models can also depend on the compatibility of their design with available computational architectures, further contributing to convergent trends.

\paragraph{Special-purpose intelligences might not converge} 
Different intelligent systems can be designed to accomplish different tasks. For instance:
A bioinformatics systems might predict protein structure; 
an autonomous vehicle might follow lanes on highways. It's possible that not much is shared between these two narrow tasks. 
Our argument only holds for intelligences that are optimized to perform well on \textit{many} tasks. We have argued that a representation of \textit{reality} is a structure that is useful across many tasks, but for any special purpose there may be shortcuts, or even effective representations detached from reality. Such shortcuts may be more efficient and necessary for continued improvements in specific domains. This will become more relevant if continued scaling comes up against boundary conditions around resources like energy and compute.

\paragraph{
How do we measure alignment?
}
We focused on one particular alignment measure, mutual nearest-neighbor, in our experiments, and cited experiments using several others. However, there is active debate on the merits and deficiencies of all these ways of measuring alignment~\citep{bansal2021revisiting, sucholutsky2023getting}. We discuss our choice and show results for other alignment metrics in Appendix \ref{sec:align-metric}.



\paragraph{Lots left to explain} We have shown results where different models arrive at \textit{similar} but not the \textit{same} representations. For example, in \Cref{fig:alignment_comparisons}, alignment clearly increases but only reaches a score of $0.16$, according to our  mutual nearest-neighbor metric. The maximum theoretical value for this metric is $1$. Is a score of $0.16$ indicative of strong alignment with the remaining gap being ``noise'' or does it signify poor alignment with major differences left to explain? We leave this as an open question.

\section*{Acknowledgements}

We thank \citeauthor{lindsey2014color} for sharing their data for our experiments shown in \Cref{fig:color_pAB}. We thank the anonymous reviewers for helpful feedback, and for providing the counterexample on how to visually convey ``I believe in the freedom of speech.'' Thanks for Yonglong Tian, Dilip Krishnan, Anna Decker, Yoon Kim, Jyo Pari, Ani Nrusimha, Dave Epstein, Victor Butoi, and Seungwook Han for helpful discussions and suggestions. We thank Mingzhong Sun for catching a typo. This work was supported by a Packard Fellowship and a Sloan Research Fellowship to P.I., by the MIT-IBM Watson AI Lab, by ONR MURI grant N00014-22-1-2740, by the Center for Brains, Minds, and Machines, the MIT Quest for Intelligence, NSF STC award CCF-1231216, the DARPA Knowledge Management at Scale and Speed (KMASS) program, and the DARPA Machine Common Sense (MCS) program.




{
\bibliographystyle{icml2024}
\bibliography{citations}
}
\clearpage

\newpage
\appendix
\onecolumn
\input{appendix}

\end{document}

%% file: appendix.tex
\appendix

\section{Mutual $k$-Nearest Neighbor Alignment Metric}\label{sec:align-metric}
For two models with representations $f$, $g$ the mutual $k$-nearest neighbor metric measures the average overlap of their respective nearest neighbor sets. In this section, we refer to this metric as $m_{\texttt{NN}}$, which we will formally define below.

For cross-modal domains, define $(x_i, y_i) \in \mathcal{X}$ as a sample from the data distribution $\mathcal{X}$ (\eg image-caption dataset). For the single domain alignment measurements, the samples are equivalent $x_i = y_i$ (\eg, images for vision, and text for language). 
Let $\{x_i, y_i\}_{i=1}^{b}$ be the corresponding mini-batch sampled from this data distribution. Then given two model representations $f$ and $g$ the corresponding features are:
$\phi_i =f(x_i)$ and $\psi_i =g(y_i)$, where the collection of these features are denoted as $\Phi = \{ \phi_1, \dots, \phi_b \}$ and $\Psi = \{ \psi_1, \dots, \psi_b \}$. 
Then for each feature pair $(\phi_i, \psi_i)$, we compute the respective nearest neighbor sets $\mathcal{S}(\phi_i)$ and $\mathcal{S}(\psi_i)$. 
\begin{align}
d_{\mathsf{knn}}(\phi_i, \Phi \setminus \phi_i) =  \mathcal{S}(\phi_i) \\
d_{\mathsf{knn}}(\psi_i, \Psi \setminus \psi_i) =  \mathcal{S}(\psi_i)
\end{align}
where $d_{\texttt{knn}}$ returns the set of indices of its $k$-nearest neighbors. Then we measure its average intersection via
\begin{align}
m_{\texttt{NN}}(\phi_i, \psi_i) = \frac{1}{k} \lvert \mathcal{S}(\phi_i) \cap \mathcal{S}(\psi_i) \rvert
\end{align}
where $\lvert {}\cdot{} \rvert$ is the size of the intersection.

\paragraph{The choice to use mutual nearest-neighbors}

Our initial efforts to measure alignment with CKA revealed a very weak trend of alignment between models, even when comparing models within their own modality. This has also been observed by~\cite{bansal2021revisiting}, which had relied on alternative metrics such as model-stitching as it ``reveals aspects of representations that measures such as centered kernel alignment (CKA) cannot''~\cite{bansal2021revisiting}.

We chose to use nearest-neighbor as a metric, as methods like CKA has a very strict definition of alignment, which may not fit our current needs. 
For instance, understanding the precise similarity between unrelated items, such as an orange and Bill Gates, may not be critical.


\paragraph{Relationship between CKA and Mutual Nearest-Neighbors}

Let $\phi_i \in \mathbb{R}^{n}$ and $\psi_i \in \mathbb{R}^{m}$ be vectorized features of two models (\eg language and vision models). Let $\bK_{ij} = \kappa(\phi_i, \phi_j)$ and $\bL_{ij} = \kappa(\psi_i, \psi_j)$ be the kernel matrices computed from a dataset using some kernel-function $\kappa$. Using an inner-product kernel, the $ij$-th entry of the centered counterpart of these Kernel matrices is:
\begin{align}
\bar{\bK}_{ij} = \langle \phi_i, \phi_j \rangle - \mathbb{E}_l[\langle \phi_i, \phi_l \rangle] \qquad\qquad \bar{\bL}_{ij} = \langle \psi_i, \psi_j \rangle - \mathbb{E}_l[\langle \psi_i, \psi_l \rangle]
\end{align}
Then, the cross-covariance of $\bK$ and $\bL$ is given by:
\begin{align}
\mathsf{HSIC}(\bK, \bL) = \frac{1}{(n-1)^2} \tr (\bar{\bK} \bar{\bL})
\end{align}
which serves as an empirical estimator of the Hilbert-Schmidt Independence Criterion~\cite{gretton2005measuring}. The Centered Kernel Alignment~(CKA)~\cite{kornblith2019similarity} is then its normalized counterpart:
\begin{align}
\mathsf{CKA}(\bK, \bL) = \frac{\mathsf{HSIC}(\bK, \bL)}{\sqrt{\mathsf{HSIC}(\bK, \bK) \mathsf{HSIC}(\bL, \bL)}}
\end{align}
CKA measures the congruence between two random variables, with a maximum alignment of $1$ and a minimum of $0$. It is invariant to isotropic scaling and offers a strict notion of alignment, measuring alignment across all samples. Hence, the CKA score reflects the global similarities of the models. This can be illustrated by expanding the trace term in HSIC:
\begin{align}
\tr(\bar{\bK} \bar{\bL}) = \sum_i \sum_j \left(\langle \phi_i, \phi_j \rangle - \mathbb{E}_l[\langle \phi_i, \phi_l \rangle]\right) \left(\langle \psi_i, \psi_j \rangle - \mathbb{E}_l[\langle \psi_i, \psi_l\rangle]\right)
\end{align}
One can modify the definition of alignment to restrict the cross-covariance measurement to samples considered to be nearest neighbors of the current sample $i$. This emphasizes similarity over dissimilarity, biasing the measure toward local alignment:
\begin{align}
\mathsf{Align_{knn}}(\bK, \bL) &= \sum_i \sum_j \alpha(i, j) \cdot \left(\langle \phi_i, \phi_j \rangle - \mathbb{E}_l[\langle \phi_i, \phi_l \rangle]\right) \left(\langle \psi_i, \psi_j \rangle - \mathbb{E}_l[\langle \psi_i, \psi_l\rangle]\right) \\
&\text{where} \qquad \alpha(i, j) = \mathbb{1}[\phi_j \in \mathsf{knn}(\phi_i) \land \psi_j \in \mathsf{knn}(\psi_i) \land i \neq j]
\end{align}

\begin{figure*}[t!]
    \centering
    \includegraphics[width=0.98\linewidth]{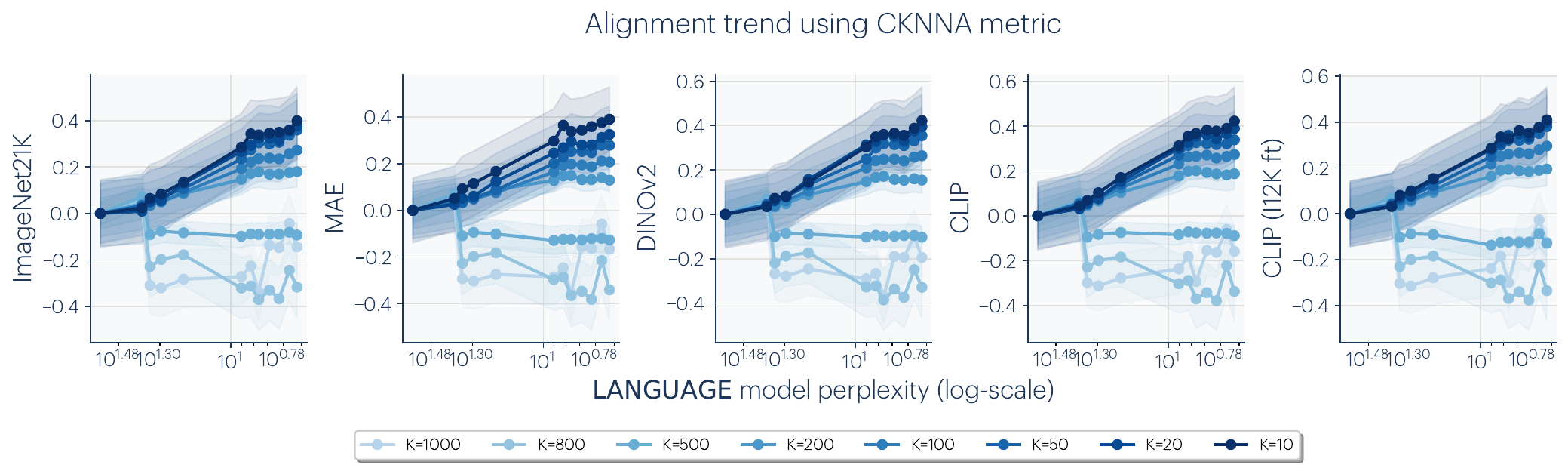}\\
    \caption{\small \textbf{Cross-modal alignment increases locally:} 
    Alignment trend when varying the top-$k$ nearest neighbors in the CKNNA metrics~(\eqn{eqn:cknna}). We center alignment score to the smallest language model and divide the total trend by the standard deviation. When $k=1024$, we recover the original CKA metric, and when $k < | \mathcal{X} |$ it closely resembles the mutual nearest-neighbor metric $m_{\texttt{NN}}$. Each line represents the average of all LLM models for a specific $k$. As we decrease $k$, the alignment becomes more pronounced.}
    \label{fig:cknna}
    \vspace{-3pt}
\end{figure*}

Where $\alpha(i, j)$ is a scalar weighting that assigns $1$ if $j$ is a mutual nearest neighbors to both $\phi_i$ and $\psi_i$, and $0$ otherwise. We refer to this metric as the Centered Kernel Nearest-Neighbor Alignment (CKNNA) metric. As the number of nearest neighbors $k \rightarrow \dim(\bK)$, we recover the original CKA metric. 

\begin{align}
\label{eqn:cknna}
\mathsf{CKNNA}(\bK, \bL) = \frac{\mathsf{Align_{knn}}(\bK, \bL)}{\sqrt{\mathsf{Align_{knn}}(\bK, \bK), \mathsf{Align_{knn}}(\bL, \bL)}}
\end{align}


We can further relax the metric to treat the cross-covariance term identically across all nearest-neighbor samples. This is equivalent to the assumption that all nearby samples have the same distance. This simplification leads us back to the mutual nearest neighbor metric:
\begin{align}
\sum_i \sum_{j} \alpha(i, j) \cdot 1 = n \cdot k \cdot m_{\texttt{NN}}(\phi_i, \psi_i)
\end{align}

By equating these metrics, we analyze the changes in alignment between language and vision models as we vary the number of neighbors $k$ in \eqn{eqn:cknna}. In \fig{fig:cknna}, we compute the average alignment score across all LLM models. For each $k$, we center the scores to the smallest vision model and divide by the standard deviation of the scores. We find that high values of $k$ show less conclusive alignment across tasks while decreasing $k$ shows a coherent trend across both models and tasks.

\newpage
\section{Consistency across various metrics}
\label{app:other-metrics}

We describe the metrics in~\tbl{tbl:metrics} and their corresponding properties. The \textit{symmetric} property implies that the metric is symmetric with respect to the data points $d(x, y) = d(y, x)$. The \textit{global} property means all samples are used to compute the distance with respect to every sample. The \textit{ordinal} property is when the ordering of the distance is taken into consideration. For example, mutual nearest neighbor is not ordinal since the nearest neighbors $\{a, b, c\}$ and $\{c, a, b\}$ are treated equally. The \textit{batchable} property is a computational property that makes it feasible to compute in a reasonable time frame.

\paragraph{Vision-vision comparison} 
In \Cref{fig:app-same-modal-rankr}, we evaluate Spearman's rank correlation among different metrics and hyperparameters over $78$ vision models (details in \Cref{sec:vision-vision-details}). We find most metrics highly correlated with each other.

\paragraph{Cross-modal comparison} We measure vision-language alignment using a range of alternative metrics.  We visualize the corresponding alignment results in~\fig{fig:metrics1of2} and~\fig{fig:metrics2of2}. Our findings indicate that alignment sensitivity not only depends on the metric used to compute it but also varies according to the specific tasks on which the vision models are trained.

\vspace{0.3in}
\begin{figure}[hbtp]
\centering
\newcommand{\centercell}[1]{\multicolumn{1}{C}{#1}}
\definecolor{lightgray}{gray}{0.9}  
\begin{tabular}{rccccp{8.2cm}} 
    \toprule
    \multirow{2}{*}{\textbf{Metric}} & \multicolumn{4}{c}{\textbf{Property}} & \thead{\multirow{2}{*}{\textbf{Description}}} \\
    \cmidrule(lr){2-6}
    & \small symmetric & \small global & \small ordinal & \small batchable & \\
    \midrule
    \rowcolor{lightgray}
    \multirow{3}{*}[-0.5em]{CKA} & 
    \multirow{3}{*}[-0.5em]{\cmark} & 
    \multirow{3}{*}[-0.5em]{\cmark} & 
    \multirow{3}{*}[-0.5em]{\cmark} & 
    \multirow{3}{*}[-0.5em]{\cmark} & 
    {\vspace{-0.5em}Centered Kernel Alignment~(CKA; \citet{kornblith2019similarity}) measures the similarity of neural networks by comparing the alignment of their kernel induced by their feature spaces.} \\[3.1em]
    \multirow{2}{*}[-0.5em]{Unbiased CKA} & 
    \multirow{2}{*}[-0.5em]{\cmark} & 
    \multirow{2}{*}[-0.5em]{\cmark} & 
    \multirow{2}{*}[-0.5em]{\cmark} & 
    \multirow{2}{*}[-0.5em]{\cmark} & 
    {\vspace{-0.5em}Unbiased estimator of CKA that corrects for sample bias in HSIC~\cite{song2012feature}.} \\[1.8em]
    \rowcolor{lightgray}
    \multirow{4}{*}[-0.5em]{SVCCA} & 
    \multirow{4}{*}[-0.5em]{\cmark} & 
    \multirow{4}{*}[-0.5em]{\cmark} & 
    \multirow{4}{*}[-0.5em]{\cmark} & 
    \multirow{4}{*}[-0.5em]{\cmark} & 
    {\vspace{-0.5em}Singular Value Canonical Correlation Analysis~(SVCCA; \citet{raghu2017svcca}) compares neural networks by decomposing their activities into singular vectors and measuring correlation.} \\[4.2em]
    \multirow{2}{*}[-0.5em]{Mutual $k$-NN} & 
    \multirow{2}{*}[-0.5em]{\cmark} & 
    \multirow{2}{*}[-0.5em]{ } & 
    \multirow{2}{*}[-0.5em]{ } & 
    \multirow{2}{*}[-0.5em]{\cmark} &
    {\vspace{-0.5em}Measures the intersection over union (IoU) of nearest neighbors between two models.} \\[1.8em]
    \rowcolor{lightgray}
    \multirow{2}{*}[-0.5em]{CKNNA} &
    \multirow{2}{*}[-0.5em]{\cmark} &
    \multirow{2}{*}[-0.5em]{\cmark$\ast$} &
    \multirow{2}{*}[-0.5em]{\cmark} &
    \multirow{2}{*}[-0.5em]{\cmark} & 
    {\vspace{-0.5em}Modified CKA measure that computes the kernel alignment only for its nearest neighbors. See~\app{sec:align-metric}.} \\[1.8em]
    \multirow{3}{*}[-0.5em]{Cycle $k$-NN} &
    \multirow{3}{*}[-0.5em]{ } &
    \multirow{3}{*}[-0.5em]{ } &
    \multirow{3}{*}[-0.5em]{ } &
    \multirow{3}{*}[-0.5em]{\cmark} & 
    {\vspace{-0.5em}Measures whether the nearest neighbor in one domain also considers the original sample as its nearest neighbor in the other domain.}  \\[3.1em]
    \rowcolor{lightgray}
    \multirow{3}{*}[-0.5em]{Edit $k$-NN} &
    \multirow{3}{*}[-0.5em]{\cmark} &
    \multirow{3}{*}[-0.5em]{\cmark$\ast$} &
    \multirow{3}{*}[-0.5em]{\cmark} &
    \multirow{3}{*}[-0.5em]{ } & 
    {\vspace{-0.5em}Computes the edit distance required to match the nearest neighbors between two datasets. The score is normalized by the maximum edit distance.} \\[3.1em]
    \multirow{2}{*}[-0.5em]{LCS $k$-NN} &
    \multirow{2}{*}[-0.5em]{\cmark} &
    \multirow{2}{*}[-0.5em]{\cmark$\ast$} &
    \multirow{2}{*}[-0.5em]{\cmark} &
    \multirow{2}{*}[-0.5em]{ } & 
    {\vspace{-0.5em}Calculates the longest common subsequence of nearest neighbors and is normalized by the sequence length.} \\[1.7em]
    \bottomrule
\end{tabular}
\caption{Comparative analysis of neural network similarity metrics. \cmark$\ast$ indicates the metric is global and still meaningful when the nearest neighbor $k$ is set to maximum batch-size $k=|\mathcal{X}|$.}
\label{tbl:metrics}
\end{figure}

\begin{figure*}[hbtp]
    \centering
    \includegraphics[width=1\linewidth]{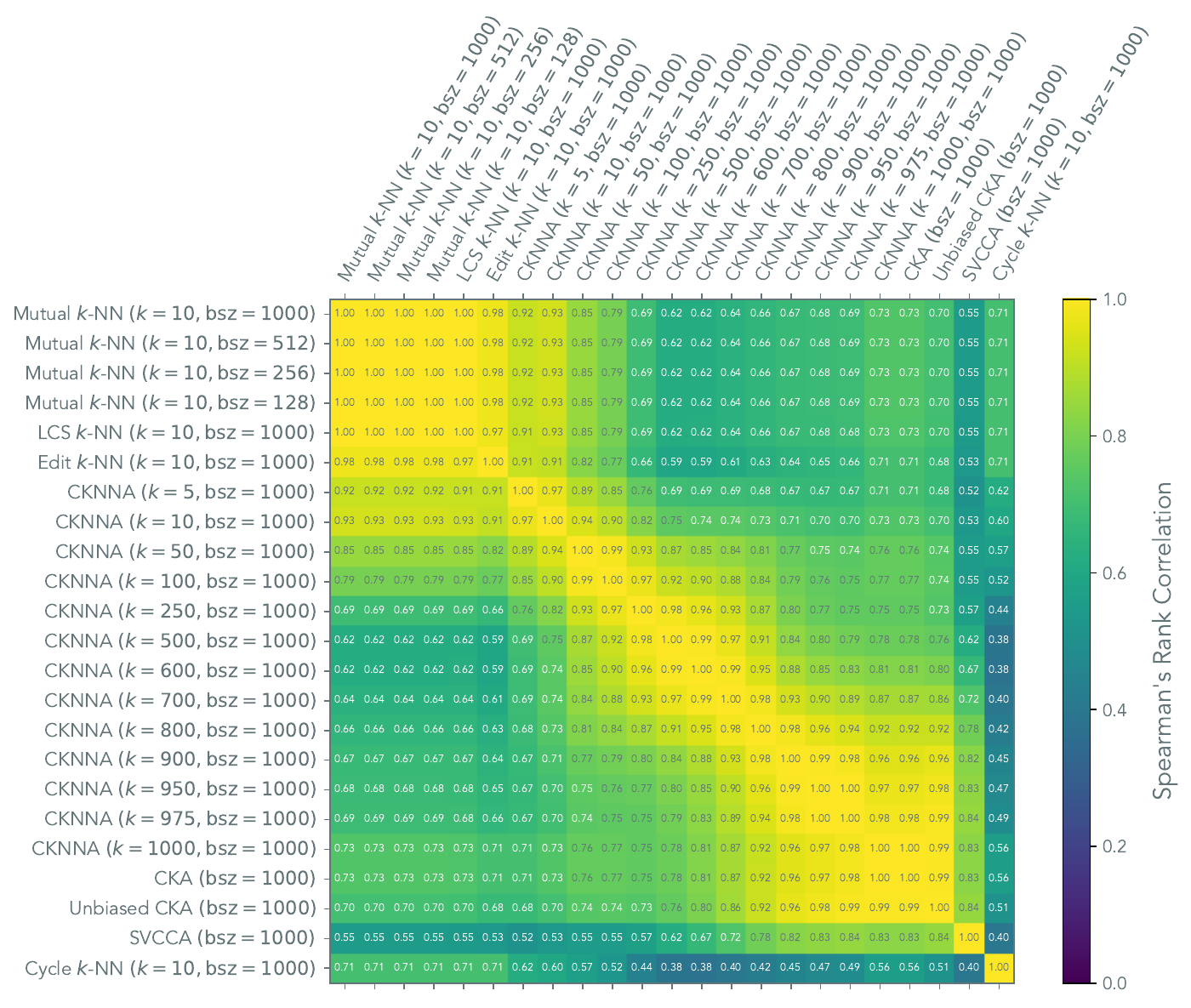}
    \caption{\textbf{Vision-vision alignment measured with various metrics.} Spearman's rank correlation among different metrics and batch sizes ($\mathsf{bsz}$) when used to measure alignment among $78$ vision models (see \Cref{sec:vision-vision-details} for details of these models). All $p$-values are below $2.24\times 10^{-105}$. Our vision-vision analysis in \Cref{fig:vm_align} is based on the first metric (Mutual $k$-NN with $k=10$ and $\mathsf{bsz}=1000$).}
    \label{fig:app-same-modal-rankr}
\end{figure*}

\begin{figure*}[hbtp]
    \centering
    \subfigure[CKA]{
        \includegraphics[width=1.0\linewidth]{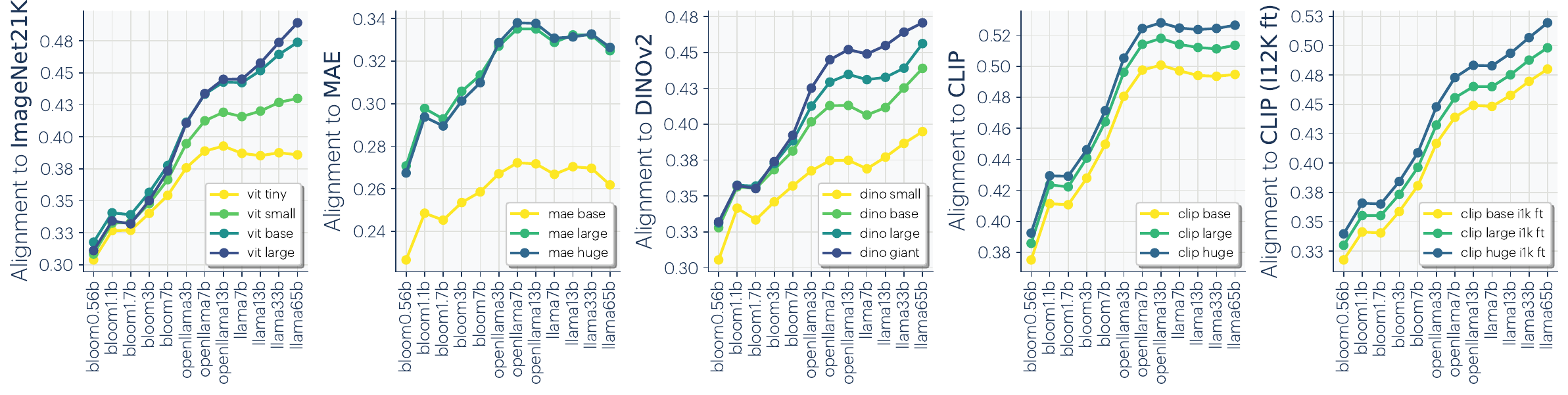}
    }\\
    \subfigure[Unbiased CKA]{
        \includegraphics[width=1.0\linewidth]{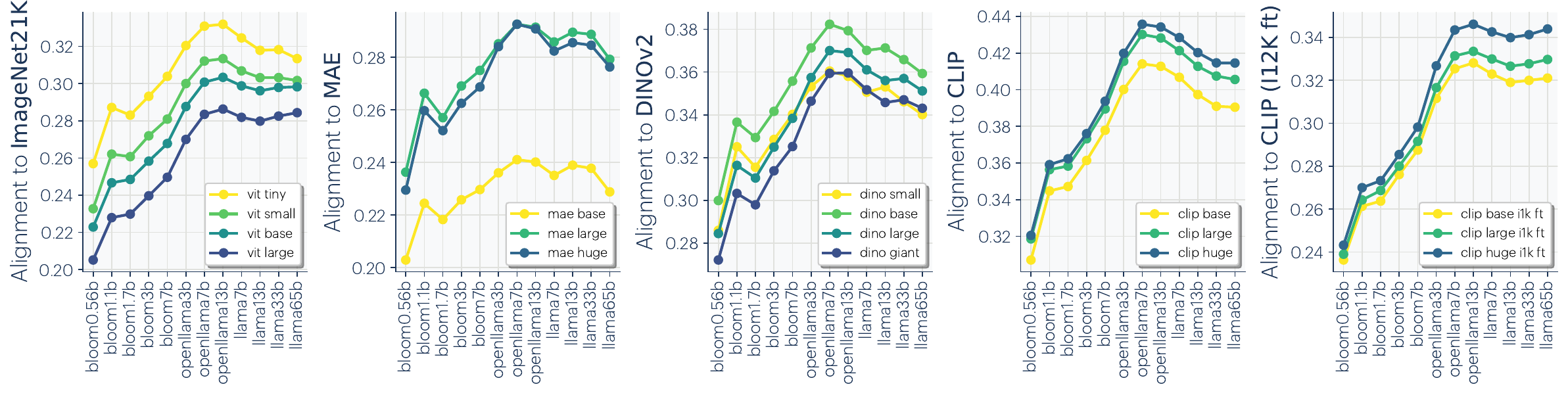}
    }\\
    \subfigure[SVCCA]{
        \includegraphics[width=1.0\linewidth]{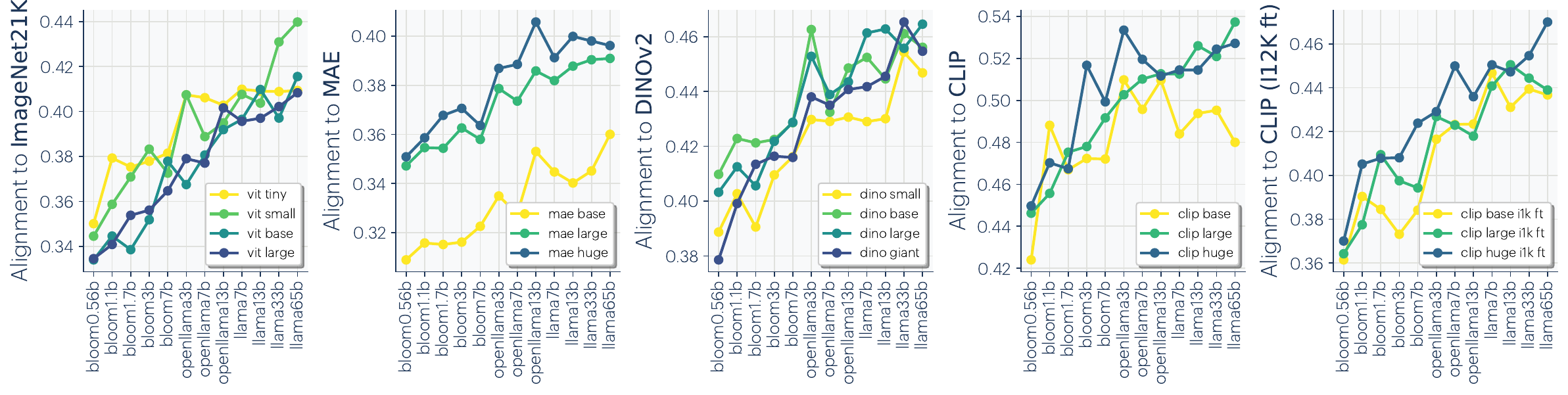}
    }\\
    \subfigure[Mutual $k$-NN ($k=10$)]{
        \includegraphics[width=1.0\linewidth]{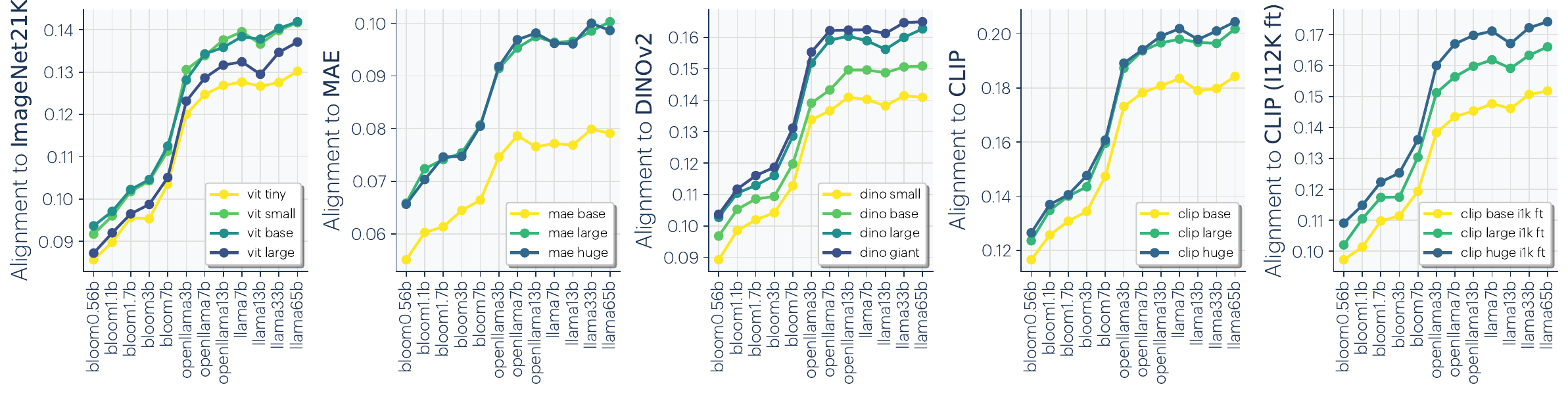}
    }
    \caption{\textbf{Cross-modal alignment for various metrics}}
    \label{fig:metrics1of2}
\end{figure*}

\begin{figure*}[hbtp]
    \centering
    \subfigure[CKNNA ($k=10$)]{
        \includegraphics[width=1.0\linewidth]{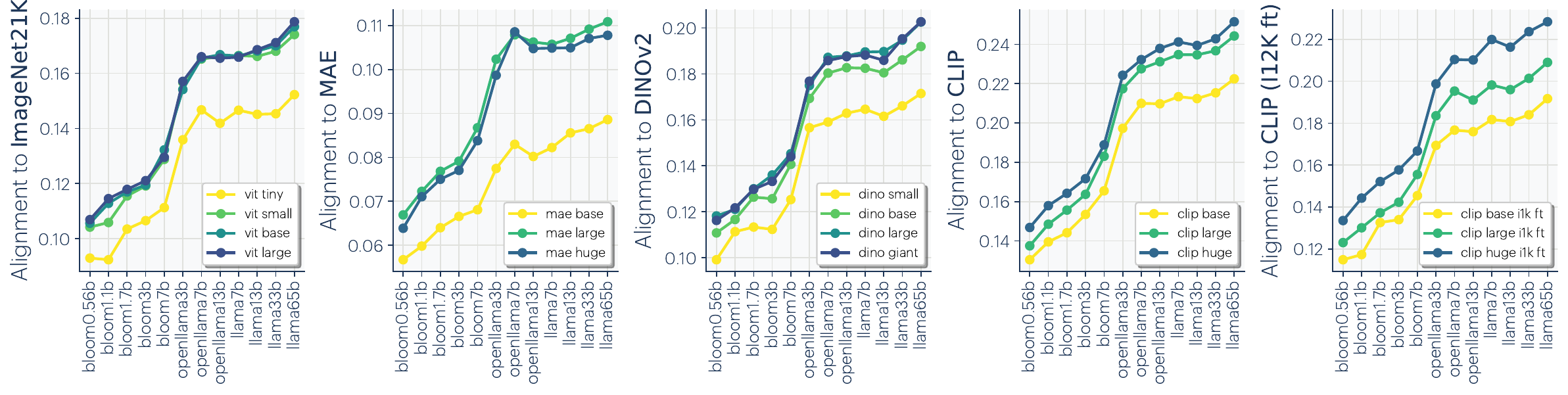}
    }
    \subfigure[Cycle $k$-NN ($k=10$)]{
        \includegraphics[width=1.0\linewidth]{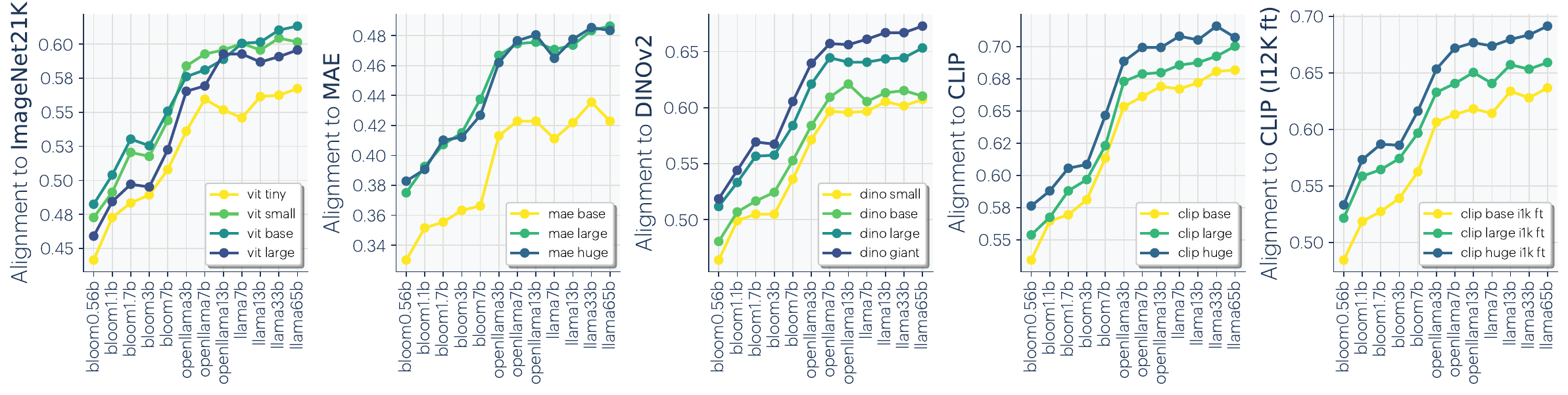}
    }\\
    \subfigure[Edit-distance $k$-NN ($k=10$)]{
        \includegraphics[width=1.0\linewidth]{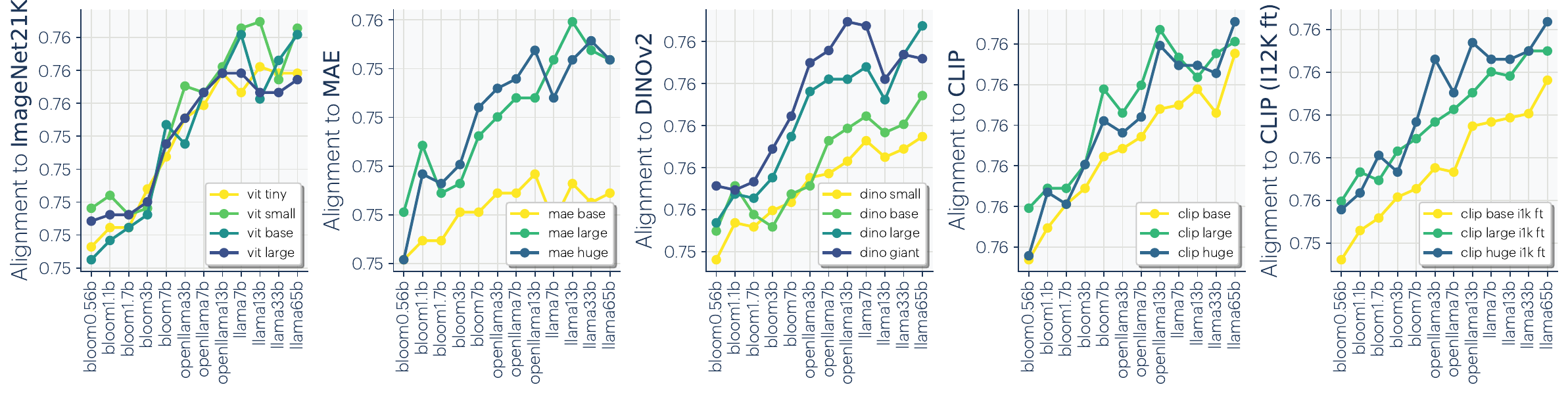}
    }\\
    \subfigure[Longest-Common-Subsequence $k$-NN ($k=10$)]{
        \includegraphics[width=1.0\linewidth]{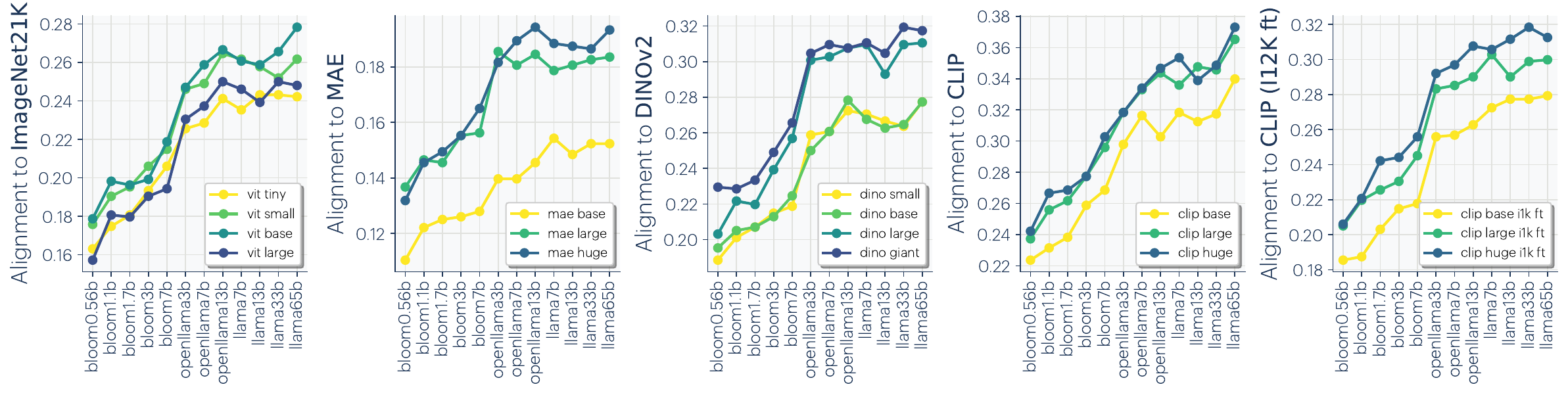}
    }
    \caption{\textbf{Cross-modal alignment measured with various metrics}}
    \label{fig:metrics2of2}
\end{figure*}

\newpage
\section{Experiments on Evaluating Alignment and Convergence}
\label{sec:alignment_methods}

To demonstrate representational convergence, we take off-the-shelf models at multiple scales and multiple modalities and measure their representational alignment. 

\subsection{Vision-Vision Alignment and Representation Quality}
\label{sec:vision-vision-details}

We consider 78 vision models in total: \begin{itemize}
    \item $17$ ViT models ranging from ViT-tiny to ViT-giant, trained on tasks including ImageNet-21k~\cite{dosovitskiy2020image} classification, Masked Autoencoders~\cite{he2021masked}, DINO~\cite{caron2021emerging}, and CLIP~\cite{radford2021learning}, including some finetuned on ImageNet-12k. 
    
    \item $1$ randomly initialized ResNet-50.
    \item $11$ ResNet-50 models trained with contrastive learning on ImageNet-1k, Places-365 \citep{zhou2017places,lopez2020semantic}, and $9$ synthetic image datasets used in \citet{baradad2022procedural}.  
    \item $49$ ResNet-18 models trained with Alignment and Uniformity contrastive loss \citep{tongzhouw2020hypersphere} on ImageNet-100, Places-365, and $47$ realistic and synthetic image datasets from \citet{baradad2021learning}.
\end{itemize}

To test representation quality, we evaluate linear probing performance on all 19 VTAB classification tasks \citep{zhai2019vtab}, which is a standard multi-task transfer learning benchmark containing structured, specialized, and natural datasets covering diverse domains. To reduce compute requirements, we subsample training and validation datasets to have at most 10{,}000 samples. We consider a representation solves a task if its performance is $\geq 80\%$ of the best performance on that task across all 78 models. 


To compute the alignment metric, we use $k=10$ nearest neighbors over $1000$ image representations computed on Places-365's validation dataset \citep{zhou2017places}. This dataset is disjoint from VTAB datasets, although both contain natural images.

\subsection{Cross-Modal Alignment}
\label{sec:vision-language-details}

We compare the representation of an image in a vision model to the representation of a caption describing that image in a language model. The language model families we consider are BLOOM~\cite{bigscience2022bloom}, OpenLLaMA~\cite{openlm2023openllama}, and LLaMA~\cite{touvron2023llama}.
For~\fig{fig:downstream}, we included more recent model families such as OLMo~\cite{groeneveld2024olmo}, LLaMA3~\cite{meta2024llama3}, Gemma~\cite{team2024gemma}, and Mistral/Mixtral~\cite{jiang2023mistral,jiang2024mixtral}. These models were downloaded from Huggingface~\cite{wolf2019huggingface}.

For vision models, we consider ViT models~\cite{dosovitskiy2020image} of various sizes trained on various data and objectives. We mainly consider the popular vision models: classification on ImageNet-21K~\cite{russakovsky2015imagenet}, MAE~\cite{he2021masked}, DINOv2~\cite{oquab2023dinov2}, CLIP~\cite{radford2021learning}, and CLIP finetuned on ImageNet-12K. These models were downloaded from PyTorch Image Models~(TIMM; \citet{timm}). This is a subset of the models used in vision-vision comparison.



To compute the alignment metric, we use $k=10$ nearest neighbors over 1024 samples from WIT (Wikipedia-based Image Text; \citet{srinivasan2021wit}). For the vision model, we use class token of each layer, and for the language model, we average pool each layer to a single token. Since it is not trivial to determine where the alignment might occur, we draw inspiration from BrainScore\cite{schrimpf2018brain} and compute pairwise alignment scores, then take the maximum. One of these pairwise comparisons also includes concatenated features. We apply $l_2$ normalization to the features before measuring the distance. As transformer architectures have ``emergent outliers''~\cite{dettmers2022gpt3}, we truncate the elements in the features that are above the $95$-th percentile.

Simply taking the last token did not show any strong alignment signal. We also experimented with prompting the language model and taking the last token representation. The prompt we used was 
\begin{align*}
    \texttt{An image with the caption `{<caption>}'. This is an image of a <fill>}
\end{align*}
Using prompting showed similar trends to average pooling but had slightly lower alignment scores.

\section{Color Cooccurrence Experiment}
\label{sec:color_cooccurrences}

Here we describe the details of how we created the four color representations visualized in \Cref{fig:color_pAB}, from left to right.

\paragraph{Perceptual representation from CIELAB color space} We embed pixels taken from the CIFAR-10 image dataset \citep{krizhevsky2009learning,torralba200880} based on the CIELAB color space, which is designed as a \emph{perceptually uniform} space that changes numerical values correspond to similar perceived changes in color.

\paragraph{Three representations from cooccurrence in VISION and LANGUAGE}
For these three representations, we first obtain a dissimilarity matrix over colors (in different ways detailed below), then use multidimensional scaling~\citep{shepard1980multidimensional} to find a 3-dimensional embedding in which Euclidean distance between the embeddings for $A$ and $B$, ${z}_A$ and ${z}_B$, best matches this dissimilarity matrix. We use $1{,}000$ fits and take the best match. Afterward, we visually align it with the CIELAB space by finding the best rotation, translation, scaling, and flipping, by running the Kabsch-Umeyama algorithm \citep{kabsch1976solution,kabsch1978discussion,umeyama1991least} twice, once on $\mathbf{z}$ and once on $-\mathbf{z}$, to account for flipping. The dissimilarity matrix we used in each case is described as following:\begin{itemize}
    \item \textbf{VISION: Pixel cooccurrence.} 
    We collect color cooccurrence statistics from the CIFAR-10 dataset, and estimate a joint distribution $p(A,B)$ over $300{,}000$ randomly sampled pixel colors $A$ and $B$ that occur within a radius of at most 4 pixels of one another. Colors are quantized on a grid in RGB space and represented as discrete variables, and $p(A,B)$ is modeled as a table of normalized counts, from which we compute the empirical pointwise mutual information matrix $\Kpmi(A, B)$. Quantization ensures that there is no bias from how color distances are represented in RGB space. Dissimilarity matrix is defined as $-\Kpmi(A, B) + c$, where $c = \max_{A, B} \Kpmi(A, B)$ is an offset to ensure non-negativity (similar to the constant in \Cref{sec:simple-contra-kpmi} and \Cref{prop:kpmi-psd-scale} that ensures neural networks can express $\Kpmi$).
    \item \textbf{LANGUAGE.} We used an approach similar to \citet{abdou2021can}.  \begin{itemize}
        \item We take $20$ pairs of (color, word) appeared in the dataset collected by \citet{lindsey2014color}, where $51$ participants were asked to free name each of the $330$ colors from the Munsell Color Chart. We filtered words that appeared less than $100$ times, and computed each word's associate color by taking the centroid in CIELAB space. Our filtering process followed \citet{abdou2021can} exactly, but resulted in $20$ colors, a slightly different set than the $18$ colors they claimed. 
        \item For each of the $20$ color words $\texttt{<col>}$, we construct three sentences: \begin{align*}
            & \texttt{The color <col>.} \\
            & \texttt{This color is <col>.} \\
            & \texttt{The color of this thing is <col>.}
        \end{align*}
        and obtain the average sentence embedding from the language encoder, as the embedding for $\texttt{<col>}$ (details below). We find this approach more effective than \citet{abdou2021can}, which uses object names that potentially have color biases, even though the objects may appear in multiple colors.
        \item Unlike \citet{abdou2021can}, we did not perform linear regression from language embedding to CIELAB space, which distorts distances and easily overfits with only $20$ samples. Instead, we used multidimensional scaling to best preserve distances, as described above.
        \item \textbf{Masked language contrastive learning (SimCSE) embedding: }
        We used sentence embedding from the unsupervised SimCSE RoBERTa-L \citep{gao2021simcse} to encode the above sentences into $1024$-dimensional embeddings, and used the pairwise Euclidean distances among $\texttt{<col>}$ embeddings as the dissimilarity matrix.
        \item \textbf{Masked language predictive learning (RoBERTa) embedding: }
        We concatenated hidden states of the last four layers of RoBERTa-L \citep{liu2019roberta}, following \citep{devlin2018bert}. We averaged across token dimensions, and obtained a $4096$-dimensional embedding for each of the above sentences, and used the pairwise Euclidean distances among $\texttt{<col>}$ embeddings as the dissimilarity matrix.
    \end{itemize}
\end{itemize}

\section{Caption Density Experiments}
\label{sec:caption_density}
We use LLaMA3-8B-Instruct~\cite{meta2024llama3} to generate summary captions at various densities for images in the Densely Captioned Images dataset~\cite{urbanek2023picture} from the train split. Following ~\citet{urbanek2023picture}, we prompt the language model with the following instructions to generate captions at differing granularity:

\texttt{system: You are given a full-text description of an image. You should summarize it into about <num\char`_words> words, being sure to include as much salient visual information as possible given the <num\char`_words> word constraint, especially information from the start of the original description. The new description should apply for the original image. Respond with only the summary, in one line.}

\texttt{user: <original\char`_caption>}

We measure the alignment with this generated caption to test our hypothesis that denser captations would result in higher alignment scores. In~\Cref{fig:caption_density}, we find that the alignment score also improves as caption length increases.

\section{Analysis of Contrastive Learners}\label{sec:analysis_contrastive}

\subsection{Contrastive objectives learn pointwise mutual information}\label{sec:analysis_contrastive-pmi}

There are two widely used forms of contrastive objectives. We now discuss each form in detail and show how they both are minimized by the pointwise mutual information (PMI) as stated in \Cref{eqn:contr-pmi}. To simplify notation, we consider learning the bivariate model $g(x_a, x_b) \in \mathbb{R}$. In \Cref{sec:what_rep},  such $g$ is optimized within the family of $\{g = \langle f_X, f_X \rangle \colon f_X \in \mathcal{F}_X\}$. 

Recall that our positive pairs are sampled from $(x, x_+) \sim \Pco$, and that the negative pairs are sampled independently from its marginals which we denote as $(x, x_-) \iidsim P$ where  $P(x) = \sum_{x_+} \Pco(x, x_+)$.
\begin{enumerate}
    \item \textbf{The binary NCE loss \citep{gutmann2010noise}} is defined with a certain prior over sampling positive vs.~negative pairs. Let $p_\mathsf{pos}$ be the probability of sampling a positive pair. Then the loss is given by \begin{equation}
        \mathcal{L}_\mathsf{binary\mbox{-}NCE}(g) 
        \trieq p_\mathsf{pos} \cdot \mathbb{E}_{(x, x_+) \sim \Pco}\left[ -\log \sigma(g(x, x_+)) \right]
        + (1 - p_\mathsf{pos}) \cdot \mathbb{E}_{(x, x_-) \iidsim P}\left[ -\log \sigma(-g(x, x_-)) \right].
    \end{equation}

    The Bayes optimal solution is given by \begin{align}
        g(x_a, x_b) 
        & = \log \frac{P(\texttt{pos} \given x_a, x_b)}{1 - P(\texttt{pos} \given x_a, x_b)} \\
        & = \log \frac{P(\texttt{pos}, x_a, x_b)}{P(\texttt{neg}, x_a, x_b)} \\
        & = \log \frac{p_\mathsf{pos} \cdot \Pco(x_a, x_b)}{(1 - p_\mathsf{pos}) P(x_a) P(x_b)} \\
        & = \log \frac{\Pco(x_a, x_b)}{P(x_a)P(x_b)} + \log \frac{p_\mathsf{pos}}{1 - p_\mathsf{pos}} \\
        & = \Kpmi(x_a, x_b) + c_X.
    \end{align}
    \item \textbf{The InfoNCE loss \citep{oord2018representation}} is defined with randomly sampling one positive pair along with $K$ negative ones. With some hyperparameter $\tau > 0$, the loss is given by \begin{equation}
        \mathcal{L}_\mathsf{InfoNCE}(g) 
        \trieq \mathbb{E}_{\substack{(x, x_+) \sim \Pco \\ (x_-^{(1)}, x_-^{(2)}, \dots, x_-^{(K)}) \iidsim P}}\left[ -\log \frac{e^{g(x, x_+) / \tau}}{e^{g(x, x_+) / \tau} + \sum_{i=1}^K e^{g(x, x_-^{(i)}) / \tau}} \right].
    \end{equation}

    The Bayes optimal solution is given by \begin{align}
        \frac{e^{g(x, x_+) / \tau}}{e^{g(x, x_+) / \tau} + \sum_{i=1}^K e^{g(x, x_-^{(i)}) / \tau}}
        & = \frac{\Pco(x_+ \given x) \prod_j P(x_-^{(j)}) }{\Pco(x_+ \given x) \prod_j P(x_-^{(j)})  + \sum_i \Pco(x_-^{(i)} \given x) P(x_+) \prod_{j \neq i} P(x_-^{(j)})} \\
        & = \frac{\Pco(x_+ \given x) / P(x_+) }{\Pco(x_+ \given x) / P(x_+) + \sum_i \Pco(x_-^{(i)} \given x) / P(x_-^{(i)})}.
    \end{align}

    For $\tau = 1$, this optima corresponds to $g$ choices where \begin{align}
        g(x_a, x_b) 
        & = \log \frac{\Pco(x_b \given x_a)}{P(x_b)} + c_X(x_a) \\
        & = \Kpmi(x_a, x_b) + c_X(x_a).
    \end{align}

    For the general $\tau \neq 1$ case, we have $g$ (and corresponding $f_X$) recovers $\Kpmi$ up to an offset and a scale. Our main argument in \Cref{sec:what_rep} that $f_X$ recovers $\Kpmi$ still holds.
\end{enumerate}

\subsection{Contrastive learners can represent $\Kpmi$ exactly under smoothness conditions}
\label{sec:analysis_contrastive-exact-repr}

We want to express $\Kpmi + C$ using some representation function $f_X \colon \mathcal{X} \rightarrow \mathbb{R}^n$ so that \begin{equation}
    \langle f_X(x_a), f_X(x_b) \rangle = \Kpmi(x_a, x_b) + C, \qquad\text{for some $C$.}
\end{equation}For such an $f_X$ to exist, an equivalent criterion is that $\Kpmi + C$ is positive semi-definite (PSD), as can be seen from eigendecomposition.

\begin{proposition}\label{prop:kpmi-psd-scale}
    Suppose that the off-diagonal elements of $\Kpmi$ are bounded within $[\log \rho_\mathsf{min}, \log \rho_\mathsf{min} + \delta] \in (-\infty, 0]$. We have $\Kpmi + C$ is positive semi-definite (PSD) for some $C$ if the joint distribution is sufficiently smooth: \begin{equation}
        \frac{\PcoiCi}{\Pcoi} \geq e^{N \delta} \rho_\mathsf{min}\mathrlap{,\qquad \forall i.}
    \end{equation}
\end{proposition}

\begin{proof}

    Note that $\Kpmi +C$ still only has non-positive off-diagonal elements if \begin{equation}
        - C \geq \log \rho_\mathsf{min} + \delta. \label{eq:C-cond}
    \end{equation}
    For such $C$, it is diagonally dominant (and thus PSD) if, \begin{equation}
        \mathllap{\forall i,\qquad} \Kii +C \geq \sum_{j\neq i} \abs{\Kij + C} = - (N - 1) C - \sum_{j\neq i} \Kij,
    \end{equation}
    or equivalently, \begin{equation}
        \mathllap{\forall i,\qquad} NC + \sum_j \Kij \geq 0.\label{eq:diag-dom-sufficient}
    \end{equation}
    
    The following choice of $C$ readily satisfies the above \Cref{eq:diag-dom-sufficient}: \begin{equation}
        C \trieq -\min_i \frac{1}{N} \sum_j \Kij. 
    \end{equation}

    Therefore, it remains to show that \Cref{eq:C-cond} is true.  Note that \begin{equation}
        - C \trieq \min_i \frac{1}{N} \sum_j \Kij \geq \frac{N-1}{N} \log \rho_\mathsf{min} + \frac{1}{N} (\min_i \Kii). 
    \end{equation}

    Therefore, it suffices to have \begin{equation}
        \log \rho_\mathsf{min} + \delta \leq \frac{N-1}{N}  \log \rho_\mathsf{min} + \frac{1}{N} (\min_i \Kii). 
    \end{equation}
    Rearranging terms gives the desired condition \begin{equation}
    \frac{\PcoiCi}{\Pcoi} \geq e^{N \delta} \rho_\mathsf{min}\mathrlap{,\qquad \forall i.}
    \end{equation}
\end{proof}
\begin{remark}
\Cref{prop:kpmi-psd-scale} is one example that a sufficiently smooth world or a sufficiently high sampling rate allows the PMI kernel $\Kpmi$ to be \emph{exactly} represented as inner products of a learned feature space (up to a scale). The condition here can be satisfied, for example, if the off-diagonal terms decay linearly with respect to $N$ and stay sufficiently close to each other. While the condition is somewhat strict, it 
captures the essence that smoothness and continuity allow easier learning. Nonetheless, we note that exact representation is not necessary for convergence, and thus this requirement can likely be relaxed.  Please see \Cref{sec:limitations} for discussions on practical settings.
\end{remark}






